\documentclass[preprint,11pt]{elsarticle}

\usepackage{amsmath,amssymb,amsfonts,amsthm,mathtools}
\usepackage{bm}
\usepackage{booktabs}
\usepackage{array}
\usepackage{graphicx}
\usepackage{placeins}
\usepackage{algorithm}
\usepackage{algorithmicx}
\usepackage{algpseudocode}
\usepackage{enumitem}
\usepackage{microtype}
\usepackage{aliascnt}
\usepackage[hypertexnames=false,hidelinks]{hyperref}
\usepackage[nameinlink,capitalize]{cleveref}
\crefname{appendix}{Appendix}{Appendices}
\crefname{algorithm}{Algorithm}{Algorithms}
\hypersetup{pdfauthor={Jiahao Zhang; Shiheng Zhang; Guang Lin}}
\myfooterfont{\fontsize{9}{11}\selectfont\itshape}

\AtBeginDocument{%
  \setlength{\abovedisplayskip}{6pt plus 2pt minus 2pt}%
  \setlength{\belowdisplayskip}{6pt plus 2pt minus 2pt}%
  \setlength{\abovedisplayshortskip}{0pt plus 2pt}%
  \setlength{\belowdisplayshortskip}{6pt plus 2pt minus 2pt}}

\graphicspath{{figures/}}

\newtheorem{theorem}{Theorem}[section]
\newaliascnt{proposition}{theorem}
\newtheorem{proposition}[proposition]{Proposition}
\aliascntresetthe{proposition}
\crefname{proposition}{Proposition}{Propositions}

\newaliascnt{corollary}{theorem}
\newtheorem{corollary}[corollary]{Corollary}
\aliascntresetthe{corollary}
\crefname{corollary}{Corollary}{Corollaries}
\theoremstyle{definition}
\newtheorem{remark}{Remark}[section]

\newcommand{\R}{\mathbb{R}}
\newcommand{\E}{\mathcal{E}}
\newcommand{\Hcal}{\mathcal{H}}
\newcommand{\diag}{\operatorname{diag}}
\newcommand{\rank}{\operatorname{rank}}
\newcommand{\odotv}{\mathbin{\odot}}
\newcommand{\norm}[1]{\left\lVert#1\right\rVert}

\begin{document}

\begin{frontmatter}

\title{A pullback-corrected scalar auxiliary variable optimizer
  with momentum and adaptive mobility}

\author[label1]{Jiahao Zhang\fnref{equal}} 

\author[label2]{Shiheng Zhang\fnref{equal}}

\fntext[equal]{These authors contributed equally to this work.}

\author[label1,label3]{Guang Lin\corref{cor1}}
\cortext[cor1]{Corresponding author (guanglin@purdue.edu)}

\affiliation[label1]{organization={School of Mechanical Engineering, Purdue University},
            city={West Lafayette, IN 47906},
            country={USA}}

\affiliation[label2]{organization={Department of Mathematics, University of Washington},
            city={Seattle, WA 98195},
            country={USA}}

\affiliation[label3]{organization={Department of Mathematics, Purdue University},
            city={West Lafayette, IN 47906},
            country={USA}}


\begin{abstract}
Objectives in scientific machine learning are often prescribed as a sum of
several terms, such as the residual, boundary, initial, and data losses of a
physics-informed neural network. In the pullback-corrected scalar auxiliary
variable (PB--SAV) method, one scalar tracks the shifted objective while the
component gradients build a positive semidefinite curvature correction of rank
at most the number of components. We carry that correction into an optimizer
with momentum and an adaptive mobility, applying it to the gradient and the
stored momentum in a single implicit solve. A mobility that is nonincreasing
in the Loewner order yields an exact modified energy law, covering Euclidean
and AMSGrad-type choices; the corresponding identity for momentum appended
after the solve carries a cross term of indefinite sign. For a fixed mobility we give a necessary and sufficient
condition for local stability at a stationary point, depending on the Hessian
minus twice the correction, and show that it also gives local geometric
convergence for every scalar relaxation sequence. The implicit solve reduces
to a dense system whose order is the number of components. In the forward
Burgers comparison, four components reduce the mean tail objective by 64.7\%
and the final solution error by 50.2\% relative to one component at the same
learning rate and momentum settings.
\end{abstract}


\end{frontmatter}

\section{Introduction}
\label{sec:introduction}

Many scientific machine learning problems have objectives with \(m\) prescribed
components,
\begin{equation}
F(\theta)=\sum_{i=1}^{m}E_i(\theta).
\label{eq:intro-objective}
\end{equation}
For example, physics-informed neural network (PINN) objectives may combine
residual, boundary, initial-condition, and data terms
\cite{raissi2019pinn,karniadakis2021physics}, with regularization added when
needed. In operator learning such as DeepONet \cite{lu2021deeponet} , a data term can likewise be separated from
weight decay. These components can have different scales and gradient
directions. Once their gradients are summed, part of this directional
information may be hidden by aggregation or cancellation. Imbalance between
the residual and boundary terms of a PINN objective, and its effect on
training, motivate the adaptive loss weighting developed for physics-informed
models \cite{wang2021gradientpathologies,wang2022pinnntk,mcclenny2023selfadaptive}.
Those methods rescale the components before they are summed. The method
developed here keeps the sum in \cref{eq:intro-objective} as it stands and
uses the component gradients to build the curvature inside the solve.

Most first-order optimization methods act primarily on the aggregated
gradient. Momentum incorporates information from previous iterations
\cite{polyak1964}, while adaptive methods such as Adam and AMSGrad introduce
coordinatewise scaling \cite{kingma2015adam,reddi2018amsgrad}. Neither
mechanism directly uses the prescribed component structure of the objective.
In an energy-dissipative scheme, momentum adds cross-iteration terms and
adaptive scaling changes the quadratic form in which those terms are measured.

The scalar auxiliary variable (SAV) approach constructs energy-stable schemes
\cite{shen2018sav,shen2019savreview}. Originally developed for dissipative
gradient flows, SAV and its relaxed variants introduce auxiliary variables so
that a modified energy satisfies a discrete dissipation law
\cite{jiang2022rsav,zhang2022gsavrelax}. These ideas have been carried to
finite-dimensional optimization \cite{liu2022,liu2023savopt}, including a
relaxed vector auxiliary variable algorithm for unconstrained problems
\cite{zhang2025rvsavopt, ZHANGop1, ZHANGop2}. Standard SAV methods,
however, primarily depend on the total gradient and therefore do not fully
exploit the decomposition in \cref{eq:intro-objective}.

The pullback-corrected SAV (PB--SAV) family for gradient flows
\cite{zhang2026pbsav} separates energy tracking from the curvature
correction. A single scalar auxiliary variable tracks the total
shifted objective, while the component gradients generate the positive
semidefinite matrix
\[
B_m
=
\sum_{i=1}^{m}
\frac{\nabla E_i(\theta)\nabla E_i(\theta)^{\mathsf T}}
{2\bigl(E_i(\theta)+C_i\bigr)}.
\]
Here each shift \(C_i\) is chosen so that \(E_i(\theta)+C_i>0\). The matrix
\(B_m\) retains the component gradient directions and has rank at most \(m\).
Its Gauss--Newton interpretation and the corresponding approximation of
quadratic Hessians are established in \cite{zhang2026pbsav}. In a direct
Euclidean step, the aggregate SAV correction only rescales the gradient, while the component pullback can also change its direction. The correction therefore uses only component energies and
first derivatives.
\Cref{sec:background} restates that construction. We build on it an
optimizer with momentum and adaptive scaling, and determine how the component
curvature enters its energy law and its local stability condition.

Our main contributions are as follows.
\begin{itemize}[leftmargin=*]

\item \textbf{Component curvature and local stability.}
\Cref{prop:block-compression} characterizes the Gauss--Newton curvature
retained by a vector-valued least-squares component. For fixed mobility,
\cref{thm:local-schur-stability,cor:local-step-interval} give the exact local
stability condition and quantify the change in admissible step sizes relative
to the aggregate correction, including the
case in which the component gradients cancel, so that the aggregate correction
vanishes while the component correction does not.

\item \textbf{Momentum inside the solve.}
The gradient and stored momentum enter one implicit equation.
\Cref{prop:inertial-model} identifies its quadratic model and the excess
incurred by appending momentum after the solve, and
\cref{prop:appended-momentum-defect} shows that the corresponding identity
for the appended construction carries a cross term of indefinite sign. For the momentum update,
\cref{thm:monotone-mobility-dissipation} establishes an exact modified energy
law for positive definite mobilities that are nonincreasing in the Loewner
order, including Euclidean and AMSGrad-type choices.

\item \textbf{Convergence of the relaxed update.}
Under the local stability condition,
\cref{thm:local-relaxed-convergence} proves local geometric convergence near a
nondegenerate minimum for fixed mobility and every scalar relaxation
sequence. For varying mobility,
\cref{thm:conditional-stationarity,cor:best-iterate-stationarity} give
stationarity and an \(O(N^{-1/2})\) best-iterate gradient bound under uniform
mobility and curvature bounds, a positive lower bound on \(q^n/Q^n\), and
relaxation parameters bounded away from one.

\end{itemize}

\Cref{sec:background} introduces SAV and the PB--SAV geometry.
\Cref{sec:methodology} presents the method and analysis; the implementation
in \cref{alg:pbsav} uses a dense solve of order at most \(m\)
(\cref{prop:low-rank-solution}). \Cref{sec:experiments} reports five numerical
experiments and \cref{sec:conclusion} concludes. Longer proofs
appear in \ref{app:convergence-proofs}. \ref{app:directional-coverage} gives a
curvature condition for exact scalar tracking under full relaxation.

\section{Preliminaries}
\label{sec:background}

\subsection{Gradient flow and the scalar auxiliary variable}

Let \(F:\R^d\to\R\) be differentiable and bounded from below. Gradient descent
is the forward Euler discretization of the Euclidean gradient flow
\(\dot\theta=-\nabla F(\theta)\), along which
\(\tfrac{\mathrm d}{\mathrm dt}F=-\norm{\nabla F}^{2}\leq0\). Choose \(C\)
with \(F(\theta)+C>0\) on the states of interest and set
\(Q(\theta)=\sqrt{F(\theta)+C}\) and \(g(\theta)=\nabla F(\theta)\). In the SAV
reformulation for gradient flows \cite{shen2018sav,shen2019savreview}, an
auxiliary scalar \(q\) tracks \(Q(\theta(t))\) and the modified energy
\(q^2\) satisfies
\(\tfrac{\mathrm d}{\mathrm dt}q^2=-\norm{\dot\theta}^{2}\leq0\).
Convergence and error analysis for the resulting schemes
\cite{shen2018convergence} and a related energy-stable variant
\cite{huang2020newsav} are available.

For the discrete formulation, let \(g^n=g(\theta^n)\) and
\(Q^n=Q(\theta^n)\), initialize \(q^0=Q(\theta^0)\), and let \(\eta>0\) be
the step size. The first-order SAV update is
\begin{subequations}
\label{eq:discrete-sav-coupled}
\begin{align}
  \frac{\theta^{n+1}-\theta^n}{\eta}
  &=-\frac{q^{n+1}}{Q^n}g^n,
  \label{eq:discrete-sav-theta}\\
  q^{n+1}-q^n
  &=\frac{(g^n)^{\mathsf T}(\theta^{n+1}-\theta^n)}{2Q^n}.
  \label{eq:discrete-sav-q}
\end{align}
\end{subequations}
Writing \(\Delta\theta^{n+1}=\theta^{n+1}-\theta^n\) and eliminating
\(q^{n+1}\) gives
\begin{equation}
  \left(\eta^{-1}I_d+B_1^n\right)\Delta\theta^{n+1}
  =-\frac{q^n}{Q^n}g^n,
  \qquad
  B_1^n=
  \frac{g^n(g^n)^{\mathsf T}}{2(Q^n)^2}.
  \label{eq:sav-eliminated}
\end{equation}
Here \(I_d\) is the \(d\times d\) identity matrix, and identities of other
sizes carry the corresponding subscript. The correction \(B_1^n\) has rank
one, and \(q^n/Q^n=1\) when the stored scalar matches \(Q(\theta^n)\).

\subsection{PB--SAV for composite objectives}

Consider an additive objective
\begin{equation}
  F(\theta)=\sum_{i=1}^{m}E_i(\theta),
  \qquad
  E_i(\theta)+C_i>0,
  \qquad
  C=\sum_{i=1}^{m}C_i.
  \label{eq:component-decomposition}
\end{equation}
Multiple-SAV methods assign a separate auxiliary variable to each energy
component \cite{cheng2018msav}. PB--SAV instead retains one auxiliary variable
for the total shifted objective and builds the correction from the individual
components \cite{zhang2026pbsav}.

At \(\theta^n\), define
\begin{equation}
  Q_i^n=\sqrt{E_i(\theta^n)+C_i},
  \qquad
  g_i^n=\nabla E_i(\theta^n),
  \qquad
  g^n=\sum_{i=1}^{m}g_i^n,
  \label{eq:component-notation-preview}
\end{equation}
so that \((Q^n)^2=\sum_{i=1}^{m}(Q_i^n)^2\). The component matrix is
\begin{equation}
  V^n=
  \begin{bmatrix}
  \dfrac{g_1^n}{\sqrt{2}Q_1^n}
  &\cdots&
  \dfrac{g_m^n}{\sqrt{2}Q_m^n}
  \end{bmatrix},
  \qquad
  B_m^n=V^n(V^n)^{\mathsf T}.
  \label{eq:pbsav-preview}
\end{equation}

\subsection{Component curvature}
\label{sec:recalled-geometry}

The PB--SAV pullback is a Gauss--Newton matrix for the transformed residuals
\cite{zhang2026pbsav}
\begin{equation}
  \psi_i(\theta)=\sqrt{2(E_i(\theta)+C_i)},
\end{equation}
with \(\psi=(\psi_1,\ldots,\psi_m)^{\mathsf T}\) and
\(F+C=\tfrac12\norm{\psi}^2\). If \(J_\psi(\theta)\) denotes its Jacobian,
differentiation gives
\begin{equation}
  B_m^n=J_\psi(\theta^n)^{\mathsf T}J_\psi(\theta^n).
  \label{eq:background-pullback-gn}
\end{equation}
For scalar least-squares components \(E_i=\tfrac12 r_i^2\) with \(C_i>0\),
this formula becomes
\begin{equation}
  B_m
  =\sum_{i=1}^m
  \frac{r_i^2}{r_i^2+2C_i}\,
  \nabla r_i\nabla r_i^{\mathsf T}.
  \label{eq:residual-pullback-curvature}
\end{equation}
Thus the component pullback approximates the Gauss--Newton matrix
\(G=\sum_i\nabla r_i\nabla r_i^{\mathsf T}\) when the shifts are small
relative to the squared residuals. For twice continuously differentiable
residuals,
\begin{equation}
  \nabla^2F=G+\sum_i r_i\nabla^2r_i,
  \label{eq:least-squares-hessian}
\end{equation}
so the residual-weighted second derivatives determine the difference between
Gauss--Newton and Hessian curvature \cite{bjorck1996least}.

The relation to the aggregate correction follows from a weighted variance
identity \cite{zhang2026pbsav}. For any \(z\in\R^d\), define
\begin{equation}
  y_i^n=\frac{z^{\mathsf T}g_i^n}{(Q_i^n)^2},
  \qquad
  \bar y^n=\frac{\sum_i(Q_i^n)^2y_i^n}{\sum_i(Q_i^n)^2}.
\end{equation}
Then
\begin{equation}
  z^{\mathsf T}(B_m^n-B_1^n)z
  =\frac12\sum_i(Q_i^n)^2(y_i^n-\bar y^n)^2\geq0.
  \label{eq:background-variance-gap}
\end{equation}
Thus \(B_m^n-B_1^n\succeq0\).

\section{PB--SAV optimization with momentum and adaptive mobility}
\label{sec:methodology}

\subsection{Component curvature and the direct PB--SAV step}
\label{sec:common-notation}

Notation follows \cref{sec:background}, with all component quantities
evaluated at \(\theta^n\). Each \(E_i\) is continuously differentiable, and
the shifted component energies are positive along the iterates. Fix
\(\eta>0\) and initialize \(q^0=Q^0\). Define the curvature gap and the
interpolation by
\begin{equation}
  S_m^n=B_m^n-B_1^n,
  \qquad
  B_\alpha^n=B_1^n+\alpha S_m^n,
  \qquad 0\leq\alpha\leq1.
  \label{eq:mixed-pullback}
\end{equation}
At \(\alpha=0\) this is the aggregate correction \(B_1^n\), and at
\(\alpha=1\) the component pullback \(B_m^n\). The variance identity
\cref{eq:background-variance-gap} gives \(S_m^n\succeq0\), so
\(B_\alpha^n\) increases with \(\alpha\) in the Loewner order. We write
\(B_1(\theta)\), \(B_m(\theta)\), and \(B_\alpha(\theta)\) for the same
matrices evaluated at an arbitrary state.

\subsubsection{Curvature retained by a least-squares component}

The scalar residual formula \cref{eq:residual-pullback-curvature} extends to
vector-valued residuals, which arise when a component contains many residual
terms.

\begin{proposition}
\label{prop:block-compression}
At a fixed parameter state, let
\(E_i=\tfrac12\norm{r_i}^2\), where
\(r_i:\R^d\to\R^{k_i}\) is continuously differentiable, and let
\(J_i=Dr_i\), \(C_i>0\), and \(G=\sum_iJ_i^{\mathsf T}J_i\). Then
\begin{equation}
  B_m=\sum_iJ_i^{\mathsf T}
  \frac{r_i r_i^{\mathsf T}}{\norm{r_i}^2+2C_i}J_i,
  \qquad 0\preceq B_m\preceq G.
  \label{eq:block-pullback}
\end{equation}
Define
\begin{equation}
  \widetilde r_i=\begin{bmatrix}r_i\\\sqrt{2C_i}\end{bmatrix},
  \qquad
  \widetilde J_i=\begin{bmatrix}J_i\\0\end{bmatrix},
  \qquad
  u_i=\frac{\widetilde r_i}{\norm{\widetilde r_i}}.
  \label{eq:augmented-block-residual}
\end{equation}
For every \(v\in\R^d\),
\begin{equation}
  v^{\mathsf T}(G-B_m)v
  =\sum_i\norm{(I_{k_i+1}-u_i u_i^{\mathsf T})\widetilde J_i v}^2.
  \label{eq:missed-block-curvature}
\end{equation}
\end{proposition}

\begin{proof}
Substitute \(g_i=J_i^{\mathsf T}r_i\) into the definition of \(B_m\).
The matrix \(u_i u_i^{\mathsf T}\) is an orthogonal projector, and
\(\widetilde J_i^{\mathsf T}\widetilde J_i=J_i^{\mathsf T}J_i\).
Subtract the projected term from that identity and evaluate on \(v\).
\end{proof}

\begin{remark}
\label{rem:retained-curvature}
Each least-squares component retains the part of its predicted residual
change parallel to its current residual, carrying the weight
\(\norm{r_i}^2/(\norm{r_i}^2+2C_i)\) of \cref{eq:block-pullback}; the
orthogonal part is discarded. At \(r_i=0\) the contribution vanishes even if
\(J_i^{\mathsf T}J_i\ne0\). For a fixed nonzero residual, the weight tends
to one as \(C_i\to0\) and decreases toward zero as \(C_i\) grows.
\end{remark}

\subsubsection{The direct update}

For an increment \(\Delta\theta^{n+1}\), define
\begin{equation}
  d^{n+1}
  =\frac{(g^n)^{\mathsf T}\Delta\theta^{n+1}}{2Q^n}.
  \label{eq:common-d}
\end{equation}
The scalar and pullback increments satisfy, with
\(\norm{x}_A^2=x^{\mathsf T}Ax\) here and below (a seminorm when the
symmetric positive semidefinite \(A\) is singular),
\begin{subequations}
\label{eq:common-identities}
\begin{align}
  (q^n+d^{n+1})^2-(q^n)^2
  &=2q^nd^{n+1}+(d^{n+1})^2,
  \label{eq:common-scalar-identity}\\
  \norm{\Delta\theta^{n+1}}_{B_\alpha^n}^2
  &=2(d^{n+1})^2
  +\alpha\norm{\Delta\theta^{n+1}}_{S_m^n}^2.
  \label{eq:common-pullback-identity}
\end{align}
\end{subequations}

The direct PB--SAV update \cite{zhang2026pbsav} is
\begin{subequations}
\label{eq:direct-family}
\begin{align}
  \left(\eta^{-1}I_d+B_\alpha^n\right)\Delta\theta^{n+1}
  &=-\frac{q^n}{Q^n}g^n,
  \label{eq:base-pbsav-step}\\
  \theta^{n+1}&=\theta^n+\Delta\theta^{n+1},
  \qquad
  q^{n+1}=q^n+d^{n+1}.
  \label{eq:base-pbsav-state}
\end{align}
\end{subequations}

\begin{theorem}
\label{thm:base-dissipation}
The update \cref{eq:direct-family} satisfies
\begin{equation}
  (q^n)^2-(q^{n+1})^2
  =\frac{\norm{\Delta\theta^{n+1}}^2}{\eta}
  +(d^{n+1})^2
  +\alpha\norm{\Delta\theta^{n+1}}_{S_m^n}^2
  \geq0.
\label{eq:base-dissipation}
\end{equation}
\end{theorem}

\begin{proof}
Take the inner product of \cref{eq:base-pbsav-step} with
\(\Delta\theta^{n+1}\).  Substitute the two identities in
\cref{eq:common-identities}, and move the change in scalar energy to the other
side. Nonnegativity follows from \(S_m^n\succeq0\).
\end{proof}

\subsection{Momentum placement and mobility}
\label{sec:monotone-mobility}

The direct update stores the scalar \(q^n\) but no previous increment. Momentum
carries previous increments into the current step. Euclidean heavy-ball momentum \cite{polyak1964} retains a fraction
\(\beta\) of the previous parameter increment \(s^n\),
\begin{equation}
  s^{n+1}=\beta s^n-\eta g^n,\qquad
  \theta^{n+1}=\theta^n+s^{n+1},\qquad 0\leq\beta<1,
  \label{eq:heavy-ball-step}
\end{equation}
and reduces to gradient descent at \(\beta=0\).

\subsubsection{Placement of momentum}

Consider first computing the direct PB--SAV increment and then appending
heavy-ball memory:
\begin{subequations}
\label{eq:appended-momentum}
\begin{align}
  \left(\eta^{-1}I_d+B_\alpha^n\right)\delta_{\mathrm{PB}}^{n+1}
  &=-\frac{q^n}{Q^n}g^n,
  \label{eq:appended-base-step}\\
  s^{n+1}&=\delta_{\mathrm{PB}}^{n+1}+\beta s^n,\\
  \theta^{n+1}&=\theta^n+s^{n+1},
  \qquad q^{n+1}=q^n+d^{n+1},
\end{align}
\end{subequations}
where \(d^{n+1}=(g^n)^{\mathsf T}s^{n+1}/(2Q^n)\). The scalar update uses
the total increment. Define the Euclidean modified energy
\begin{equation}
  \Hcal_{\mathrm{app}}^n=(q^n)^2+\frac{\norm{s^n}^2}{2\eta}.
  \label{eq:appended-energy}
\end{equation}

\begin{proposition}
\label{prop:appended-momentum-defect}
The update \cref{eq:appended-momentum} satisfies
\begin{equation}
\begin{aligned}
  \Hcal_{\mathrm{app}}^n-\Hcal_{\mathrm{app}}^{n+1}
  ={}&\frac{\norm{s^{n+1}-\beta s^n}^2}{2\eta}
  +\frac{1-\beta^2}{2\eta}\norm{s^n}^2\\
  &+(d^{n+1})^2+\alpha\norm{s^{n+1}}_{S_m^n}^2
  -\beta(s^{n+1})^{\mathsf T}B_\alpha^n s^n.
\end{aligned}
\label{eq:appended-momentum-defect}
\end{equation}
The last term has no fixed sign.
\end{proposition}

\begin{proof}
The total increment obeys
\((\eta^{-1}I_d+B_\alpha^n)(s^{n+1}-\beta s^n)=-(q^n/Q^n)g^n\).
Take its inner product with \(s^{n+1}\) and apply
\cref{eq:common-identities}. Subtracting the change in kinetic energy and
using Euclidean polarization gives \cref{eq:appended-momentum-defect}.
\end{proof}

\subsubsection{The momentum and the mobility}

We place momentum on the right-hand side of the implicit solve, so that the
pullback acts on the gradient and memory together. This removes the extra
\(\beta B_\alpha^n s^n\) term in the appended state equation.
Let \(p^n\in\R^d\) be the momentum and let the mobility
\(M^n\in\R^{d\times d}\), symmetric positive definite, map it to the
parameter increment \(M^np^n\); this is the mobility in gradient-flow
terminology and a variable-metric preconditioner in optimization. The new
momentum and increment are related by
\begin{equation}
  \Delta\theta^{n+1}=M^{n+1}p^{n+1},
  \qquad
  p^{n+1}=(M^{n+1})^{-1}\Delta\theta^{n+1}.
  \label{eq:covector-map}
\end{equation}
We assume that the mobility is nonincreasing in the Loewner order:
\begin{equation}
  0\prec M^{n+1}\preceq M^n.
  \label{eq:abstract-mobility-order}
\end{equation}
The PB--SAV momentum update is
\begin{equation}
  \left((M^{n+1})^{-1}+\eta B_\alpha^n\right)
  \Delta\theta^{n+1}
  =\beta p^n-\eta\frac{q^n}{Q^n}g^n,
  \label{eq:abstract-momentum-step}
\end{equation}
or, equivalently,
\begin{equation}
  \frac{p^{n+1}-\beta p^n}{\eta}
  +B_\alpha^n\Delta\theta^{n+1}
  +\frac{q^n}{Q^n}g^n=0.
  \label{eq:abstract-momentum-balance}
\end{equation}
With Euclidean mobility, omitting the pullback term and setting \(q^n/Q^n=1\)
in this balance gives the heavy-ball recursion \cref{eq:heavy-ball-step},
with \(p^n=s^n\).
The parameter and scalar states remain
\begin{equation}
  \theta^{n+1}=\theta^n+\Delta\theta^{n+1},
  \qquad
  q^{n+1}=q^n+d^{n+1}.
  \label{eq:abstract-momentum-states}
\end{equation}

\subsubsection{The inertial quadratic model}
\label{sec:momentum-filtering}

The implicit solve \cref{eq:abstract-momentum-step} has a quadratic
optimization interpretation.

\begin{proposition}
\label{prop:inertial-model}
The increment \(\Delta\theta^{n+1}\) solving
\cref{eq:abstract-momentum-step} is the unique minimizer of
\begin{equation}
  \mathcal Q(s)
  =\frac{q^n}{Q^n}(g^n)^{\mathsf T}s
  +\frac12 s^{\mathsf T}B_\alpha^ns
  +\frac{1}{2\eta}\norm{s-\beta M^{n+1}p^n}_{(M^{n+1})^{-1}}^2 .
  \label{eq:inertial-quadratic-model}
\end{equation}
Appending the momentum after the solve gives instead
\begin{equation}
  \Delta_{\mathrm{app}}^{n+1}
  =-\eta\frac{q^n}{Q^n}
    \bigl((M^{n+1})^{-1}+\eta B_\alpha^n\bigr)^{-1}g^n
   +\beta M^{n+1}p^n,
  \label{eq:frozen-appended-comparator}
\end{equation}
and the two increments satisfy
\begin{align}
  \Delta_{\mathrm{app}}^{n+1}-\Delta\theta^{n+1}
  &=\beta\eta\bigl((M^{n+1})^{-1}+\eta B_\alpha^n\bigr)^{-1}
    B_\alpha^nM^{n+1}p^n,
  \label{eq:appended-model-step-defect}\\
  \mathcal Q(\Delta_{\mathrm{app}}^{n+1})-\mathcal Q(\Delta\theta^{n+1})
  &=\frac{1}{2\eta}
  \norm{\Delta_{\mathrm{app}}^{n+1}-\Delta\theta^{n+1}}
       _{(M^{n+1})^{-1}+\eta B_\alpha^n}^2\geq0.
  \label{eq:appended-model-excess}
\end{align}
\end{proposition}

\begin{proof}
The Hessian of \(\mathcal Q\) is
\(B_\alpha^n+\eta^{-1}(M^{n+1})^{-1}\succ0\), and its optimality equation is
\cref{eq:abstract-momentum-step}. Subtracting the two increments and using
\(\bigl((M^{n+1})^{-1}+\eta B_\alpha^n\bigr)M^{n+1}
 =I_d+\eta B_\alpha^nM^{n+1}\) gives
\cref{eq:appended-model-step-defect}. Expanding \(\mathcal Q\) about its
minimizer gives \cref{eq:appended-model-excess}.
\end{proof}

\begin{remark}
\label{rem:inertial-gauss-newton}
For a fixed \(\alpha\), let \(\Psi_\alpha(\theta)\) stack
\(\sqrt{2(1-\alpha)}\,Q(\theta)\) on \(\sqrt{\alpha}\,\psi(\theta)\), so that
\(\tfrac12\norm{\Psi_\alpha}^2=F+C\), its Gauss--Newton matrix is
\(B_\alpha\), and its gradient is \(g\). At exact
tracking \(q^n=Q^n\), the shifted model \(\mathcal Q(s)+F+C\) is a
Gauss--Newton model for the residual map \(\Psi_\alpha\), together with a
proximal term centered at \(\beta M^{n+1}p^n\) and measured in
\(\norm{\cdot}_{(M^{n+1})^{-1}}\); otherwise the linear term is scaled by
\(q^n/Q^n\). Levenberg--Marquardt methods for neural networks also combine curvature with
adaptive momentum \cite{pooladzandi2022improving}, without the scalar
coupling that produces the energy law below.
\end{remark}

To see how the solve acts on memory, let \(b_j\geq0\) be the eigenvalues of
\((M^{n+1})^{1/2}B_\alpha^n(M^{n+1})^{1/2}\) and let \(u_j\) be the
coordinates of \((M^{n+1})^{1/2}p^n\) in a corresponding orthonormal
eigenbasis. In those coordinates the solve multiplies the \(j\)th component
of the stored momentum by \(\beta/(1+\eta b_j)\), and
\begin{equation}
  \mathcal Q(\Delta_{\mathrm{app}}^{n+1})-\mathcal Q(\Delta\theta^{n+1})
  =\frac{\beta^2\eta}{2}
  \sum_j\frac{b_j^2}{1+\eta b_j}u_j^2 .
  \label{eq:spectral-memory-coefficient}
\end{equation}

\begin{remark}
\label{rem:memory-damping}
Directions with larger \(b_j\) damp the stored momentum more strongly, and
where \(b_j=0\) the coefficient remains \(\beta\). The excess in
\cref{eq:spectral-memory-coefficient} is measured in the quadratic model, not
in the modified energy of \cref{thm:monotone-mobility-dissipation}.
\end{remark}

\subsubsection{Curvature action of the momentum step}
\label{sec:curvature-action}

The following comparison measures the approximation error of the correction
through its action on the reference increment.

\begin{proposition}
\label{prop:curvature-action}
Let \(M\succ0\), let \(A,B\succeq0\) be symmetric, and let \(\eta>0\).
For a common right-hand side \(f\in\R^d\), define
\begin{equation}
  (M^{-1}+\eta B)\Delta_B=f,
  \qquad
  (M^{-1}+\eta A)\Delta_A=f.
  \label{eq:curvature-reference-steps}
\end{equation}
Then
\begin{equation}
  \Delta_B-\Delta_A
  =\eta(M^{-1}+\eta B)^{-1}(A-B)\Delta_A,
  \label{eq:curvature-action-identity}
\end{equation}
and
\begin{equation}
  \norm{\Delta_B-\Delta_A}_{M^{-1}}
  \leq\eta\norm{M^{1/2}(A-B)\Delta_A}.
  \label{eq:curvature-action-bound}
\end{equation}
In particular, \(\Delta_B=\Delta_A\) if and only if
\((A-B)\Delta_A=0\).
\end{proposition}

\begin{proof}
Subtracting the two systems gives \cref{eq:curvature-action-identity}.
Since \(M^{1/2}BM^{1/2}\succeq0\), we have
\(\norm{(I_d+\eta M^{1/2}BM^{1/2})^{-1}}_2\leq1\); multiplying
\cref{eq:curvature-action-identity} by \(M^{-1/2}\) then gives
\cref{eq:curvature-action-bound}. The equality criterion follows from
invertibility of \(M^{-1}+\eta B\).
\end{proof}

\begin{remark}
\label{rem:curvature-action}
For the momentum update, take \(M=M^{n+1}\), \(B=B_\alpha^n\), and
\(f=\beta p^n-\eta(q^n/Q^n)g^n\); a Gauss--Newton matrix, or a positive
semidefinite Hessian at the current state, can serve as \(A\). A low-rank \(B\)
approximates \(\Delta_A\) when the weighted action error on the right-hand
side of \cref{eq:curvature-action-bound} is small, even when \(A-B\) is not
small in operator norm, and reproduces it exactly when
\((A-B)\Delta_A=0\). The quadratic and regression tests
in \cref{sec:quadratic-experiment,sec:regression-experiment} measure the
matrix error and the step error separately.
\end{remark}

\subsection{Dissipation and mobility choices}
\label{sec:dissipation-mobilities}

\subsubsection{Modified energy law}

Define, for every \(n\),
\begin{equation}
  \Hcal_{\mathrm{mom}}^n
  =(q^n)^2+\frac{1}{2\eta}\norm{p^n}_{M^n}^2.
  \label{eq:abstract-momentum-energy}
\end{equation}

\begin{theorem}
\label{thm:monotone-mobility-dissipation}
For \(0\leq\beta<1\) and \(0\prec M^{n+1}\preceq M^n\), the update
\cref{eq:covector-map,eq:abstract-momentum-step,eq:abstract-momentum-states}
satisfies
\begin{equation}
\begin{aligned}
  \Hcal_{\mathrm{mom}}^n
  -\Hcal_{\mathrm{mom}}^{n+1}
  ={}&\frac{\norm{p^n}_{M^n-M^{n+1}}^2}{2\eta}
  +\frac{\norm{p^{n+1}-\beta p^n}_{M^{n+1}}^2}{2\eta}\\
  &+\frac{1-\beta^2}{2\eta}\norm{p^n}_{M^{n+1}}^2
  +(d^{n+1})^2
  +\alpha\norm{\Delta\theta^{n+1}}_{S_m^n}^2
  \geq0.
\end{aligned}
\label{eq:monotone-mobility-dissipation}
\end{equation}
\end{theorem}

The proof is a weighted polarization of
\cref{eq:abstract-momentum-balance} against
\(\Delta\theta^{n+1}=M^{n+1}p^{n+1}\), followed by substitution of
\cref{eq:common-identities}; it is given in \cref{app:energy-law-proof}.

\begin{remark}
\label{rem:dissipated-quantity}
The dissipated quantity is the modified energy \(\Hcal_{\mathrm{mom}}^n\),
which controls the scalar and kinetic energies together rather than the
objective \(F(\theta^n)\) alone. Its five nonnegative terms are, in order,
the dissipation from the mobility change, the inertial residual, the momentum
damping, the scalar tracking term, and the curvature gap, which the
relaxation step of \cref{sec:relaxation-convergence} sums.
\end{remark}

\subsubsection{Euclidean momentum}

Set \(M^n=I_d\) and identify \(p^n=s^n\).  The abstract state equation becomes
\begin{equation}
  \left(I_d+\eta B_\alpha^n\right)s^{n+1}
  =\beta s^n-\eta\frac{q^n}{Q^n}g^n.
  \label{eq:euclidean-momentum-step}
\end{equation}
The constant mobility \(M^n\equiv I_d\) satisfies
\cref{eq:abstract-mobility-order} trivially, so
\cref{thm:monotone-mobility-dissipation} applies for every
\(0\leq\beta<1\); its mobility-change term vanishes, and the modified
energy is \((q^n)^2+\norm{s^n}^2/(2\eta)\).

\subsubsection{AMSGrad-type diagonal mobility}

An AMSGrad-type maximum envelope \cite{reddi2018amsgrad} gives a diagonal
mobility that is nonincreasing in the Loewner order. Let \(0\leq\beta_2<1\),
\(\epsilon>0\), and initialize
\(v^0=\overline v^0=0\).  With \(\odotv\) the componentwise product and the
square root, division, and maximum below taken componentwise, define
\begin{subequations}
\label{eq:adaptive-second-moment}
\begin{align}
  v^{n+1}
  &=\beta_2v^n+(1-\beta_2)(g^n\odotv g^n),
  \label{eq:adaptive-v}\\
  \widehat v^{n+1}
  &=\frac{v^{n+1}}{1-\beta_2^{n+1}},
  \label{eq:adaptive-bias-correction}\\
  \overline v^{n+1}
  &=\max\{\overline v^n,\widehat v^{n+1}\},
  \label{eq:adaptive-max}\\
  M^{n+1}
  &=\diag\left(
  \frac{1}{\sqrt{\overline v^{n+1}}+\epsilon}
  \right).
  \label{eq:adaptive-mobility}
\end{align}
\end{subequations}
Then \(M^0=\epsilon^{-1}I_d\), and the maximum envelope implies
\begin{equation}
  0\prec M^{n+1}\preceq M^n,
  \label{eq:adaptive-mobility-order}
\end{equation}
so \cref{thm:monotone-mobility-dissipation} applies to this mobility for
every \(0\leq\beta<1\).

\subsection{Scalar relaxation}
\label{sec:relaxation-convergence}

Relaxed-SAV schemes modify the auxiliary state while preserving a discrete
energy bound \cite{jiang2022rsav,zhang2022gsavrelax,huang2024optimalrsav}.
The squared-scalar update below is the relaxation of \cite{zhang2026pbsav};
it depends on the update only through its Lyapunov functional, so one
statement covers both updates of this section.

Without relaxation, \(q^{n+1}=q^n+d^{n+1}\) is the stored scalar in
\cref{eq:direct-family,eq:abstract-momentum-states}. When relaxation is
applied, the same algebraic output is provisional and is denoted by
\begin{equation}
  \bar q^{n+1}=q^n+d^{n+1}.
  \label{eq:relaxation-candidate}
\end{equation}
For the direct update the old and provisional energies are
\(\Hcal_{\mathrm{dir}}^n=(q^n)^2\) and
\(\bar\Hcal_{\mathrm{dir}}^{n+1}=(\bar q^{n+1})^2\); for the momentum update
they are \(\Hcal_{\mathrm{mom}}^n=(q^n)^2+\norm{p^n}_{M^n}^2/(2\eta)\) and
\(\bar\Hcal_{\mathrm{mom}}^{n+1}=(\bar q^{n+1})^2
+\norm{p^{n+1}}_{M^{n+1}}^2/(2\eta)\).

\begin{proposition}
\label{prop:generic-relaxation}
For either update, let \((\Hcal^n,\bar\Hcal^{n+1})\) denote the corresponding
pair above and suppose that
\begin{equation}
  \Hcal^n-\bar\Hcal^{n+1}=\mathcal D^n\geq0.
  \label{eq:generic-unrelaxed-dissipation}
\end{equation}
For \(0\leq\rho^n\leq1\), define
\begin{equation}
  (q^{n+1})^2
  =\min\left\{
  Q(\theta^{n+1})^2,\;
  (\bar q^{n+1})^2+\rho^n\mathcal D^n
  \right\},
  \qquad q^{n+1}\geq0.
  \label{eq:generic-relaxation}
\end{equation}
Let \(\Hcal^{n+1}\) be the same energy evaluated with the accepted
scalar \(q^{n+1}\). Then
\begin{equation}
  \Hcal^{n+1}
  \leq\Hcal^n-(1-\rho^n)\mathcal D^n
  \leq\Hcal^n.
  \label{eq:relaxed-energy-bound}
\end{equation}
\end{proposition}

\begin{proof}
Relaxation changes only the scalar part of the new energy.  Therefore
\(
\Hcal^{n+1}\leq\bar\Hcal^{n+1}+\rho^n\mathcal D^n
\), and \cref{eq:generic-unrelaxed-dissipation} gives
\cref{eq:relaxed-energy-bound}.
\end{proof}

With \(q^0=Q(\theta^0)\), the accepted scalars satisfy
\(0\leq q^n\leq Q(\theta^n)\) by \cref{eq:generic-relaxation}.
Full relaxation \(\rho^n=1\) preserves the modified energy bound; choosing
\(\rho^n\leq\bar\rho<1\) also gives
\(\sum_n\mathcal D^n\leq\Hcal^0/(1-\bar\rho)\).
Under the curvature condition in \ref{app:directional-coverage}, full
relaxation preserves \(q^n=Q(\theta^n)\), so the momentum energy becomes
\(F(\theta^n)+C+\norm{p^n}_{M^n}^2/(2\eta)\).

\subsection{Low-rank implementation and cost per iteration}
\label{sec:low-rank-realization}

With \(V^n\) as in \cref{eq:pbsav-preview} and
\(c^n=\tfrac1{Q^n}[\,Q_1^n\ \cdots\ Q_m^n\,]^{\mathsf T}\), we have
\(\norm{c^n}=1\) and
\begin{equation}
  B_\alpha^n
  =V^n\left[\alpha I_m+(1-\alpha)c^n(c^n)^{\mathsf T}\right](V^n)^{\mathsf T}.
  \label{eq:true-rank-factor}
\end{equation}
Consequently \(\rank(B_\alpha^n)\leq\rank(V^n)\leq m\).  For a concrete
factor with at most \(m\) columns, define
\begin{equation}
  R^n=\left[\alpha I_m+(1-\alpha)c^n(c^n)^{\mathsf T}\right]^{1/2},
  \qquad
  W^n=V^nR^n.
  \label{eq:explicit-low-rank-factor}
\end{equation}
Because \(\norm{c^n}=1\), the symmetric positive semidefinite square root has
the closed form
\(R^n=\sqrt{\alpha}\,I_m+(1-\sqrt{\alpha})c^n(c^n)^{\mathsf T}\), which at
\(\alpha=0\) is the rank-one projector \(c^n(c^n)^{\mathsf T}\).  It also
satisfies
\(W^n(W^n)^{\mathsf T}=B_\alpha^n\) with \(\rank(W^n)\leq m\).  With
\begin{equation}
  f^n=\beta p^n-\eta\frac{q^n}{Q^n}g^n,
  \label{eq:low-rank-rhs}
\end{equation}
the momentum equation is
\begin{equation}
  \left((M^{n+1})^{-1}+\eta W^n(W^n)^{\mathsf T}\right)
  \Delta\theta^{n+1}=f^n.
  \label{eq:low-rank-system}
\end{equation}

\begin{algorithm}[!ht]
\caption{Relaxed PB--SAV momentum with nonincreasing mobility. Use
\(M^n\equiv I_d\) for Euclidean momentum; for the AMSGrad-type mobility use
\(v^0=\overline v^0=0\) and \(M^0=\epsilon^{-1}I_d\).}
\label{alg:pbsav}
\begin{algorithmic}[1]
\Require Components \(\{E_i,C_i\}_{i=1}^m\), \(\theta^0\),
\(q^0=Q(\theta^0)\), \(p^0=0\), number of updates \(N\), \(\eta>0\),
\(\beta\in[0,1)\), \(\alpha\in[0,1]\), relaxation schedule
\(\{\rho^n\}_{n=0}^{N-1}\subset[0,1]\), and a mobility initialization and
update rule satisfying \(0\prec M^{n+1}\preceq M^n\)
\For{\(n=0,\ldots,N-1\)}
  \State Evaluate \(E_i(\theta^n)\), \(g_i^n\), \(Q_i^n\), \(Q^n\), and
  \(g^n=\sum_i g_i^n\).
  \State Update \(M^{n+1}\) using the prescribed nonincreasing mobility rule.
  \State Form \(W^n\) from \cref{eq:explicit-low-rank-factor}.
  \State Form \(f^n\) from \cref{eq:low-rank-rhs}.
  \State Compute \(\Delta\theta^{n+1}\) from \cref{eq:general-smw}.
  \State Set \(p^{n+1}=(M^{n+1})^{-1}\Delta\theta^{n+1}\) and
  \(\theta^{n+1}=\theta^n+\Delta\theta^{n+1}\).
  \State Compute \(\bar q^{n+1}\) from \cref{eq:relaxation-candidate}.
  \State Sum the five terms in
  \cref{eq:monotone-mobility-dissipation} as \(\mathcal D^n\).
  \State Re-evaluate \(Q(\theta^{n+1})\) and apply
  \cref{eq:generic-relaxation}.
\EndFor
\Ensure \(\theta^N\)
\end{algorithmic}
\end{algorithm}

\begin{proposition}
\label{prop:low-rank-solution}
The unique solution of \cref{eq:low-rank-system} is
\begin{equation}
\begin{aligned}
  \Delta\theta^{n+1}
  ={}&M^{n+1}f^n
  -\eta M^{n+1}W^n
  \left(I_m+\eta (W^n)^{\mathsf T}M^{n+1}W^n\right)^{-1}\\
  &\hspace{35mm}\times
  (W^n)^{\mathsf T}M^{n+1}f^n.
\end{aligned}
\label{eq:general-smw}
\end{equation}
Besides applying the mobility, the computation requires a dense system of
order at most \(m\).
\end{proposition}

\begin{proof}
The coefficient in \cref{eq:low-rank-system} is positive definite because
\(M^{n+1}\succ0\) and \(W^n(W^n)^{\mathsf T}\succeq0\).  The
Sherman--Morrison--Woodbury identity \cite{sherman1950,woodbury1950} with
base inverse \(M^{n+1}\) gives \cref{eq:general-smw}.
\end{proof}

\begin{remark}
\label{rem:cost}
Each iteration evaluates the \(m\) component energies and gradients that the
algorithm already forms.  The factor \(R^n\) is available in closed form, so
\[
  W^n=\sqrt{\alpha}\,V^n
  +\bigl(1-\sqrt{\alpha}\bigr)\bigl(V^nc^n\bigr)(c^n)^{\mathsf T}
\]
costs one matrix--vector product and needs no eigensolver.
Evaluating \cref{eq:general-smw} requires the Gram matrix
\((W^n)^{\mathsf T}M^{n+1}W^n\) and one dense solve of order at most \(m\),
so the work beyond applying the mobility is \(O(m^2d+m^3)\).  Storage beyond
the iterate is the \(d\times m\) factor \(W^n\), the momentum \(p^n\), and
the mobility state.  The relaxation step re-evaluates \(Q(\theta^{n+1})\),
one objective evaluation per iteration beyond the gradient.
\end{remark}

The direct update omits the momentum and mobility states: it solves
\(\left(I_d+\eta W^n(W^n)^{\mathsf T}\right)\Delta\theta^{n+1}
=f_{\mathrm{dir}}^n\) with
\(f_{\mathrm{dir}}^n=-\eta\frac{q^n}{Q^n}g^n\), using the same reduction to a
dense system of order at most \(m\), and takes \(\mathcal D^n\) in the scalar
relaxation from \cref{eq:base-dissipation}.

\subsection{Local stability and convergence}
\label{sec:local-stability-convergence}

\subsubsection{Local stability and retained curvature}
\label{sec:local-stability}

With fixed mobility, the pullback changes the local stability condition at a
stationary point \(\theta_*\) of \(F\). Throughout this subsection a
subscript \(*\) denotes evaluation at that point, so that
\(H_*=\nabla^2F(\theta_*)\) and \(B_*=B_\alpha(\theta_*)\), and
\((\theta_*,p_*,q_*)\) is the corresponding state of
\cref{eq:abstract-momentum-step,eq:abstract-momentum-states}. The comparison
depends on the Hessian curvature left after subtracting twice \(B_*\).

\begin{theorem}
\label{thm:local-schur-stability}
Let \(\theta_*\) satisfy \(\nabla F(\theta_*)=0\), suppose the component
energies are twice continuously differentiable near \(\theta_*\) with
positive shifted energies, and let \(H_*\succ0\). Fix
\(0\leq\alpha\leq1\), \(M\succ0\), \(\eta>0\), and
\(0\leq\beta<1\). The corresponding rest state has \(p_*=0\) and a stored
scalar \(q_*>0\); set \(\kappa_*=q_*/Q(\theta_*)\). The parameter--momentum block of the
linearization of \cref{eq:abstract-momentum-step,eq:abstract-momentum-states}
is Schur stable, meaning that all its eigenvalues have modulus less than one,
if and only if
\begin{equation}
  \kappa_*H_*
  \prec 2B_*+\frac{2(1+\beta)}{\eta}M^{-1}.
  \label{eq:local-stability-criterion}
\end{equation}
\end{theorem}

The proof, in \cref{app:schur-proof}, reduces the linearized
parameter--momentum map to a matrix polynomial and then to a scalar quadratic
in the Rayleigh quotients \(b=\xi^*\widehat B_*\xi/(\xi^*\xi)\) and
\(h=\xi^*\widehat H_*\xi/(\xi^*\xi)\), where
\(\widehat B_*=M^{1/2}B_*M^{1/2}\), \(\widehat H_*=M^{1/2}H_*M^{1/2}\),
and \(\xi^*\) is the conjugate transpose of \(\xi\in\mathbb C^d\).

The unrelaxed scalar equation has a neutral direction at the stationary state;
the scalar is controlled separately in the local convergence argument below.

\begin{corollary}
\label{cor:local-step-interval}
Under the assumptions of \cref{thm:local-schur-stability}, define
\begin{equation}
  L_{\mathrm{eff}}
  =\lambda_{\max}\!\left(
  M^{1/2}(\kappa_*H_*-2B_*)M^{1/2}\right).
  \label{eq:effective-local-stiffness}
\end{equation}
If \(L_{\mathrm{eff}}>0\), the exact stability interval is
\begin{equation}
  0<\eta<\frac{2(1+\beta)}{L_{\mathrm{eff}}}.
  \label{eq:exact-local-step-interval}
\end{equation}
If \(L_{\mathrm{eff}}\leq0\), every finite \(\eta>0\) is admissible.
At the stationary point,
\begin{equation}
  B_1(\theta_*)=0,
  \qquad B_*=\alpha B_m(\theta_*),
  \qquad \rank B_m(\theta_*)\leq\min\{d,m-1\}.
  \label{eq:stationary-component-curvature}
\end{equation}
For a common \(\theta_*\), \(q_*\), \(M\), and \(\beta\), increasing
\(\alpha\) cannot shrink the stability interval relative to the aggregate
choice \(\alpha=0\).
\end{corollary}

\begin{proof}
Apply a congruence by \(M^{1/2}\) to
\cref{eq:local-stability-criterion}. The total gradient vanishes at
\(\theta_*\), giving \(B_1(\theta_*)=0\). The component gradients sum
to zero and are therefore linearly dependent, which gives the rank bound.
Finally, \(B_m\succeq0\) makes \(L_{\mathrm{eff}}\) nonincreasing with
\(\alpha\).
\end{proof}

\begin{remark}
\label{rem:stationary-cancellation}
Gradient cancellation can leave nonzero component curvature after the
aggregate correction has vanished; if all component gradients vanish, so does
the component correction. Since \(H_*\succ0\), stability for every
\(\eta>0\) requires \(B_*\succ0\), which \cref{eq:stationary-component-curvature}
excludes once \(m-1<d\); a low-rank correction can nevertheless enlarge a
finite interval. The comparison in \cref{cor:local-step-interval} is between
component and aggregate pullbacks within the momentum update;
\cref{app:appended-comparison} records the corresponding comparison with the
frozen appended recurrence, whose stability interval is at least as large.
\end{remark}

\subsubsection{Local convergence under any scalar relaxation}
\label{sec:local-relaxed-convergence}

\begin{theorem}
\label{thm:local-relaxed-convergence}
Let the component energies be twice continuously differentiable near
\(\theta_*\), with positive shifted energies,
\(\nabla F(\theta_*)=0\), and \(H_*\succ0\). Fix
\(M\succ0\), \(\eta>0\), \(0\leq\beta<1\), and
\(0\leq\alpha\leq1\), and assume
\begin{equation}
  H_*\prec2B_\alpha(\theta_*)+
  \frac{2(1+\beta)}{\eta}M^{-1}.
  \label{eq:local-relaxed-stability-condition}
\end{equation}
Use the momentum update with \(M^n\equiv M\), the relaxation
\cref{eq:generic-relaxation} with any sequence \(\rho^n\in[0,1]\), and
\(q^0=Q(\theta^0)\). There exist \(\delta>0\), \(c_0,c_1<\infty\),
and \(r\in(0,1)\) such that, whenever
\(a_0=\norm{\theta^0-\theta_*}+\norm{p^0}<\delta\), the iterates remain
near \(\theta_*\) and
\begin{equation}
  \norm{\theta^n-\theta_*}+\norm{p^n}\leq c_0r^na_0.
  \label{eq:local-geometric-convergence}
\end{equation}
Moreover, \(q^n\to q_\infty>0\), with
\begin{equation}
  0\leq Q(\theta_*)-q_\infty\leq c_1a_0^2.
  \label{eq:local-scalar-limit}
\end{equation}
The constants are uniform over all such relaxation sequences, including
\(\rho^n\equiv1\).
\end{theorem}

The proof in \cref{app:local-convergence-proof} combines contraction of the
parameter--momentum map with a quadratic bound on the scalar drift.

\subsection{Conditional stationarity}
\label{sec:conditional-stationarity}

For a varying mobility, a global stationarity criterion follows from
relaxation parameters bounded away from one and from a ratio \(q^n/Q^n\)
bounded below by a positive constant. Assume
\begin{equation}
  0\leq\rho^n\leq\bar\rho<1.
  \label{eq:strict-relaxation-budget}
\end{equation}
The proofs of the following estimates are given in
\cref{app:global-stationarity-proofs}.

\begin{theorem}
\label{thm:conditional-stationarity}
Consider the relaxed momentum update
\cref{eq:covector-map,eq:abstract-momentum-step,eq:generic-relaxation} with a
fixed \(0\leq\beta<1\).  Suppose there exist constants
\begin{equation}
  0<\mu_{\min}\leq\mu_{\max}<\infty,
  \qquad \underline\kappa>0,
  \label{eq:stationarity-constants}
\end{equation}
such that, for every \(n\),
\begin{equation}
  \mu_{\min}I_d\preceq M^n\preceq\mu_{\max}I_d,
  \qquad
  \frac{q^n}{Q^n}\geq\underline\kappa,
  \qquad
  0\prec M^{n+1}\preceq M^n.
  \label{eq:stationarity-assumptions}
\end{equation}
Assume also that \(\sup_n\norm{B_\alpha^n}<\infty\) and
\cref{eq:strict-relaxation-budget} holds.  Then
\begin{equation}
  p^n\to0,
  \qquad
  p^{n+1}-\beta p^n\to0,
  \qquad
  \Delta\theta^{n+1}\to0,
  \qquad
  g^n\to0.
  \label{eq:stationarity-limits}
\end{equation}
Consequently, every accumulation point of \(\{\theta^n\}\) is stationary.
\end{theorem}

\begin{remark}
If \(E_i\geq0\), \(C_i>0\), and
\(\sup_n\norm{g_i^n}\leq\Gamma_i\), then
\begin{equation}
  \norm{B_\alpha^n}
  \leq \norm{B_m^n}
  \leq \sum_{i=1}^m\frac{\Gamma_i^2}{2C_i}.
  \label{eq:bounded-pullback-condition}
\end{equation}
For the diagonal realization \cref{eq:adaptive-mobility},
\(\sup_n\norm{g^n}_\infty\leq\Gamma_\infty\) gives
\begin{equation}
  \frac{1}{\Gamma_\infty+\epsilon}I_d\preceq M^n\preceq\frac1\epsilon I_d.
  \label{eq:bounded-mobility-condition}
\end{equation}
The lower bound on \(q^n/Q^n\) must still be assumed separately.
\end{remark}

\begin{corollary}
\label{cor:direct-stationarity}
Consider the relaxed direct update
\cref{eq:direct-family,eq:generic-relaxation}.  Suppose
\cref{eq:strict-relaxation-budget} holds,
\(\sup_n\norm{B_\alpha^n}<\infty\), and there exists \(\underline\kappa>0\)
such that \(q^n/Q^n\geq\underline\kappa\) for every \(n\).  Then
\(\Delta\theta^{n+1}\to0\), \(g^n\to0\), and every accumulation point is
stationary.
\end{corollary}

\begin{corollary}
\label{cor:best-iterate-stationarity}
Let
\(
L_B=\sup_n\norm{B_\alpha^n}
\).
Under the assumptions of \cref{thm:conditional-stationarity} or of
\cref{cor:direct-stationarity}, there is a constant \(K\), given
explicitly in \cref{app:global-stationarity-proofs}, depending for the
momentum update only on \(\eta\), \(\beta\), \(\underline\kappa\), the
mobility bounds, and \(L_B\), and for the direct update only on \(\eta\),
\(\underline\kappa\), and \(L_B\).  For either update, let \(\Hcal^0\) be its
initial modified energy.  Then, for every \(N\geq1\),
\begin{equation}
  \sum_{n=0}^{N-1}\norm{g^n}^2
  \leq \frac{K\Hcal^0}{1-\bar\rho},
  \qquad
  \min_{0\leq n<N}\norm{g^n}^2
  \leq \frac{K\Hcal^0}{(1-\bar\rho)N}.
  \label{eq:best-iterate-stationarity}
\end{equation}
Consequently, \(\sum_{n=0}^\infty\norm{g^n}^2<\infty\), and the best
gradient norm among the first \(N\) iterates is \(O(N^{-1/2})\).
\end{corollary}

Full relaxation admits a different sufficient condition: the pullback must
cover the true objective remainder along each step, up to a fixed fraction
of the inertial dissipation. Under this condition, and under the same uniform
mobility and pullback bounds as \cref{thm:conditional-stationarity},
\cref{thm:coverage-full-relaxation} gives exact tracking and stationarity
without a separate lower bound on \(q^n/Q^n\).

\section{Numerical experiments}
\label{sec:experiments}

The experiments examine two aspects of component curvature: how accurately the
component pullback reproduces Newton or Gauss--Newton increments, and whether
resolving the objective into components improves practical optimization. The
quadratic and regression tests compare curvature matrices and the increments
they produce at shared states, as in \cref{prop:curvature-action}. The forward
Burgers experiment compares one and four components with matched learning rates
and momentum settings. Darcy operator learning and inverse Burgers
identification assess the complete optimizer of \cref{alg:pbsav} with its
selected learning rates.

All calculations use double precision and deterministic full batches.  The
neural-network experiments were performed with PyTorch 2.5.1 and CUDA 12.4 on
an NVIDIA A100 80GB GPU.  Within each paired comparison, all methods start from
the same parameters and use the same training, validation, and test data.
Hyperparameters and model choices use only the training or validation
quantities specified below; test errors are reported after selection.  Curves
show all recorded iterates without smoothing.  The representative Darcy field
is selected by a criterion independent of optimizer output; the inverse Burgers
panel uses the median PB--SAV seed by final viscosity error, so that no
best-case run is displayed.  The Darcy and Burgers results use five paired
seeds and report the seed mean with a band of one sample standard deviation.

Heavy-ball gradient descent applies the weight-decay coefficient as coupled
\(L^2\) regularization and AdamW applies it as decoupled weight decay
\cite{loshchilov2019adamw}; because these mechanisms are not identical, each
comparison reports the objective that is common to all methods.

\subsection{Quadratic problem: Newton increments without full Hessian recovery}
\label{sec:quadratic-experiment}

We use the quadratic benchmark with two curvature classes from
\cite{zhang2026pbsav} and add a comparison of full Hessian error with the
error in the damped Newton increment:
\begin{equation}
  F(\phi)=\sum_{j=1}^{100}a_j\phi_j^2,
  \qquad
  a_{2k-1}=1,
  \qquad
  a_{2k}=10^{-2}.
  \label{eq:quadratic-problem}
\end{equation}
Its Hessian is \(H=2\diag(a_1,\ldots,a_{100})\), and the minimizer is
\(\phi_*=0\).  The two distinct values of \(a_j\)
form two curvature classes.  We test whether a
low-rank component split can reproduce the damped Newton increment even when it
does not recover the full matrix \(H\).

All methods start from \(\phi^0=\bm 1\in\R^{100}\) and run for 300 updates.
PB--SAV uses
\(\eta=10\), and the shift associated with each coordinate energy is
\(10^{-12}\).  The one-component formulation aggregates all coordinates.  The
two-component formulation separates the two curvature classes.  The
100-component formulation assigns one energy component to each coordinate.
The intermediate component counts \(4,8,16,32,64\) form a nested refinement
and are used in the fixed-state geometry comparison.
All PB--SAV runs in this subsection use the direct update with \(\alpha=1\).

Gradient descent uses the analytically optimal constant rate
\(2/(\lambda_{\min}(H)+\lambda_{\max}(H))=0.990099\).
The reference direction is the exact damped Newton increment
\begin{equation}
  \Delta_{\rm DN}
  =-(\eta^{-1}I_d+H)^{-1}\nabla F(\phi).
  \label{eq:damped-newton-step}
\end{equation}
The matrix error is measured in the relative Frobenius norm, and the step
error is measured relative to \(\Delta_{\rm DN}\) in the Euclidean norm.  Both
are evaluated at the shared initial state.

\begin{figure}[t]
  \centering
  \includegraphics[width=0.78\textwidth]{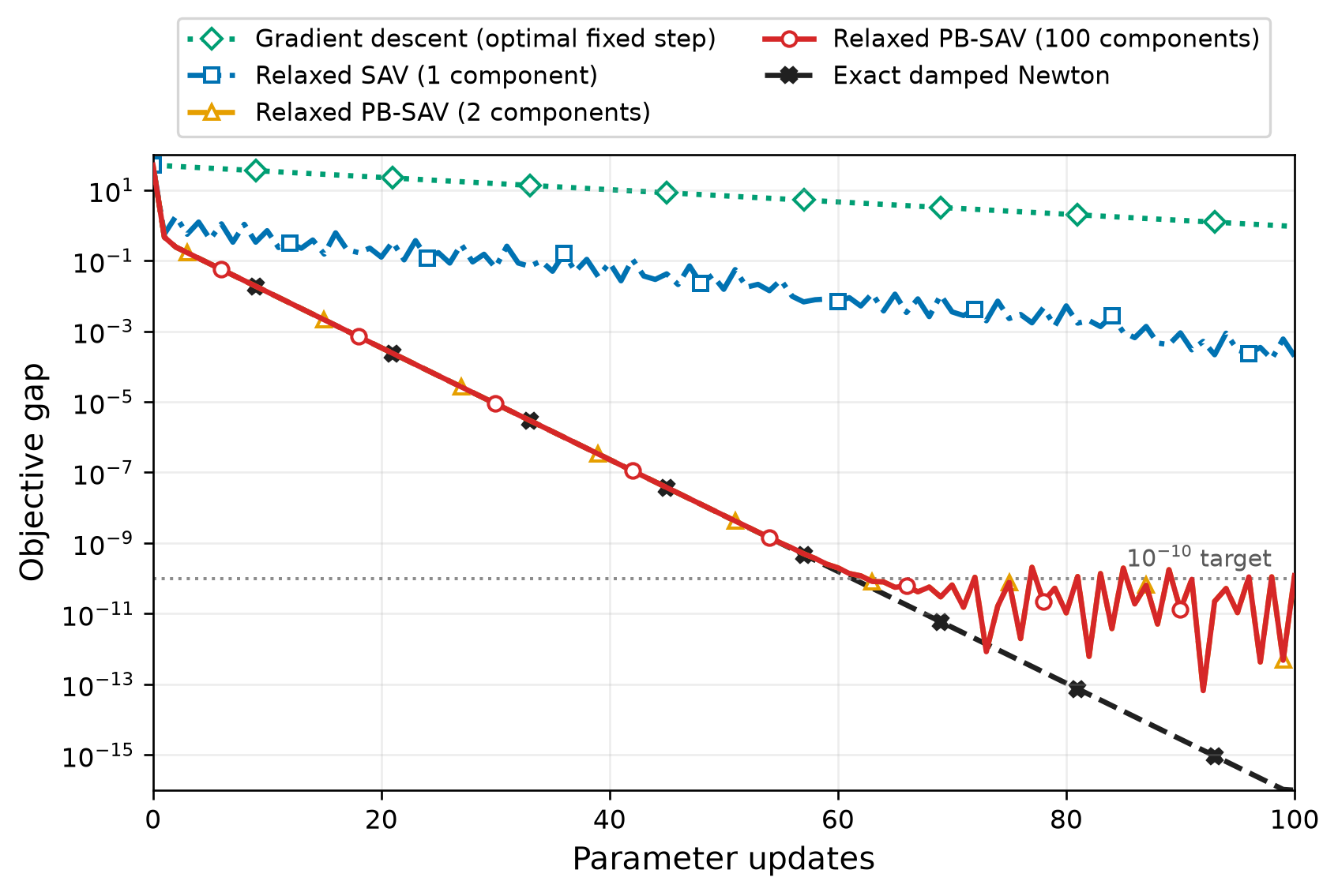}
  \caption{Convergence on the quadratic problem during the first 100 updates.
  All methods start from the same state.}
  \label{fig:quadratic-convergence}
\end{figure}

\begin{figure}[t]
  \centering
  \includegraphics[width=\textwidth]{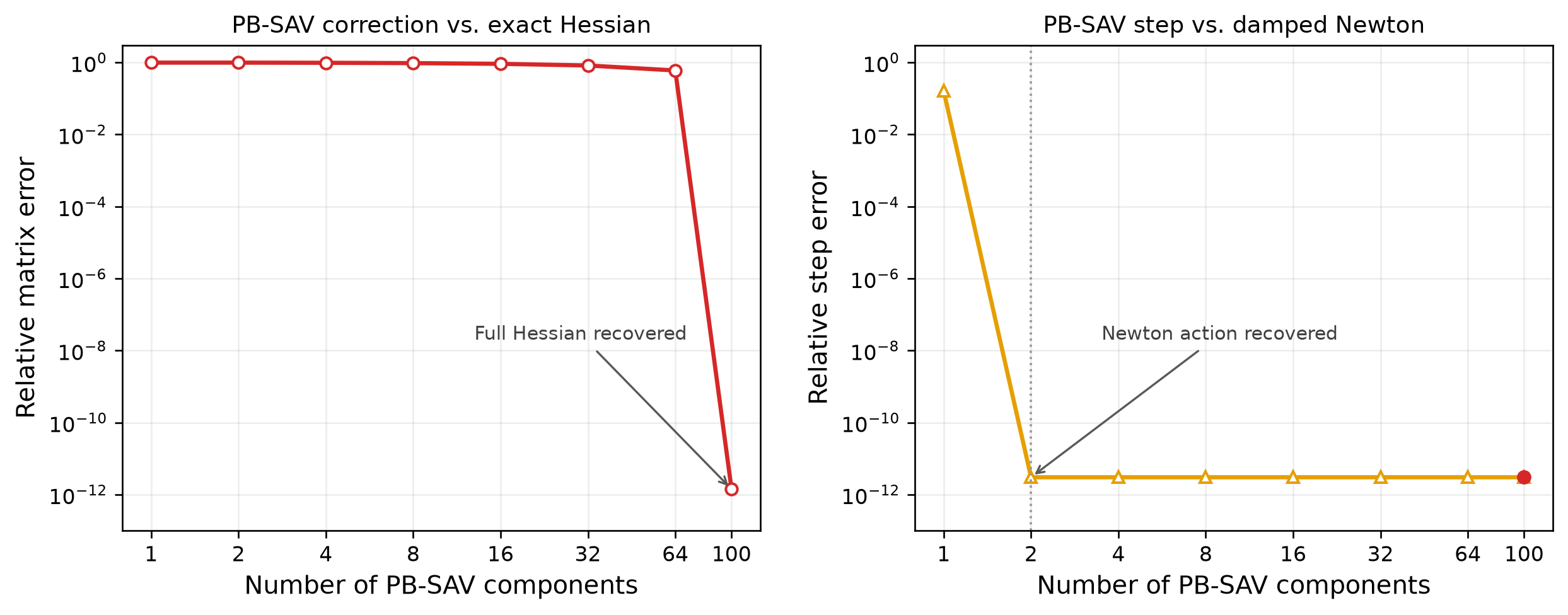}
  \caption{Fixed-state geometry for the quadratic problem.  Left: relative
  error between the PB--SAV correction and the exact Hessian.  Right: relative
  error between the PB--SAV increment and the damped Newton increment.}
  \label{fig:quadratic-recovery}
\end{figure}

\begin{table}[t]
\centering
\caption{Quadratic comparison after 300 updates.  The Hessian and step errors
are evaluated at the common initial state.  The final gap is
\(F(\phi^{300})-F(\phi_*)\), with \(F(\phi_*)=0\).  A dash indicates
that the quantity is not applicable or that the target was not reached within
300 updates.  The columns report different quantities and are not ranked, so no
entry is bold.}
\label{tab:quadratic-results}
\footnotesize
\setlength{\tabcolsep}{2pt}
\begin{tabular}{@{}lcccc@{}}
\toprule
Method & \shortstack{Hessian\\error} & \shortstack{Newton-step\\error} &
\shortstack{Updates to\\\(10^{-10}\)} & \shortstack{Final\\gap} \\
\midrule
Gradient descent (optimal step) & -- & -- & -- & \(3.102\times10^{-4}\) \\
Relaxed SAV (1 comp.) & \(9.900\times10^{-1}\) & \(1.627\times10^{-1}\) & -- & \(1.873\times10^{-10}\) \\
Relaxed PB--SAV (2 comp.) & \(9.899\times10^{-1}\) & \(3.022\times10^{-12}\) & 63 & \(7.879\times10^{-11}\) \\
Relaxed PB--SAV (100 comp.) & \(1.414\times10^{-12}\) & \(3.022\times10^{-12}\) & 63 & \(9.623\times10^{-11}\) \\
Exact damped Newton & -- & 0 & 62 & \(1.550\times10^{-48}\) \\
\bottomrule
\end{tabular}
\end{table}

The fixed-state comparison in \cref{fig:quadratic-recovery} separates the
error in the matrix from the error in the increment it produces.  The
two-component correction has a relative Hessian error of
\(9.899\times10^{-1}\), but its step error is only \(3.022\times10^{-12}\).  The 100-component correction
recovers the full Hessian to \(1.414\times10^{-12}\), yet it produces the same
step accuracy as the two-component formulation.  The trajectories in
\cref{fig:quadratic-convergence,tab:quadratic-results} follow the same
pattern.  The two PB--SAV variants reach the \(10^{-10}\) target in 63
updates, compared with 62 updates for exact damped Newton.  One-component SAV
converges more slowly, and optimally tuned fixed-step gradient descent remains
above the target after 300 updates.

At \(\phi^0\), the gradient lies in the two-dimensional subspace spanned by the
indicator vectors of the two curvature classes.  The two-component pullback
agrees with the damped Hessian on this subspace, but not outside it.

\subsection{Nonlinear regression: Gauss--Newton recovery}
\label{sec:regression-experiment}

The second example tests the pullback interpretation in a nonlinear
least-squares problem.  We fit \(N=30\) noisy observations of
\(f_*(x)=\sin(2\pi x)\) for \(x\in[0,1]\),
where the observation noise is independent Gaussian noise with standard
deviation \(0.1\).  The model is a one-hidden-layer tanh network,
\begin{equation}
  f_\theta(x)=\sum_{j=1}^{50}c_j\tanh(a_jx+b_j)+b_0,
  \label{eq:regression-network}
\end{equation}
with 151 trainable parameters.  The model is therefore overparameterized
relative to the 30 observations.  The objective is
\begin{equation}
  F(\theta)=\frac{1}{2N}\sum_{i=1}^{N}
  \left(f_\theta(x_i)-y_i\right)^2.
  \label{eq:regression-objective}
\end{equation}
For a pointwise least-squares decomposition,
\cref{eq:residual-pullback-curvature} relates the PB--SAV pullback to the
Gauss--Newton matrix, with scaled residuals
\(r_i=(f_\theta(x_i)-y_i)/\sqrt{N}\). This example tests the relation at
both the matrix and trajectory levels.

Nested random partitions use \(1,2,4,8,16\), and 30 components.  The final
partition assigns one squared residual to each observation.  PB--SAV and
damped Gauss--Newton use the same adaptive damping controller, with
\(10^{-6}\leq\eta\leq100\), and the same Armijo acceptance test
\cite{armijo1966} on the true objective. Gradient descent uses the rate
\(0.22\), selected by a paired
200-update training-objective pilot.  Adam is piloted over
\(10^{-3}\), \(3\times10^{-3}\), \(10^{-2}\), and \(3\times10^{-2}\),
and selects \(3\times10^{-2}\).  Each reported run contains 600 accepted
updates.  Prediction is measured by the mean-squared error (MSE) against the
noise-free target.
The PB--SAV runs use the direct update with \(\alpha=1\).

\begin{figure}[t]
  \centering
  \includegraphics[width=0.80\textwidth]{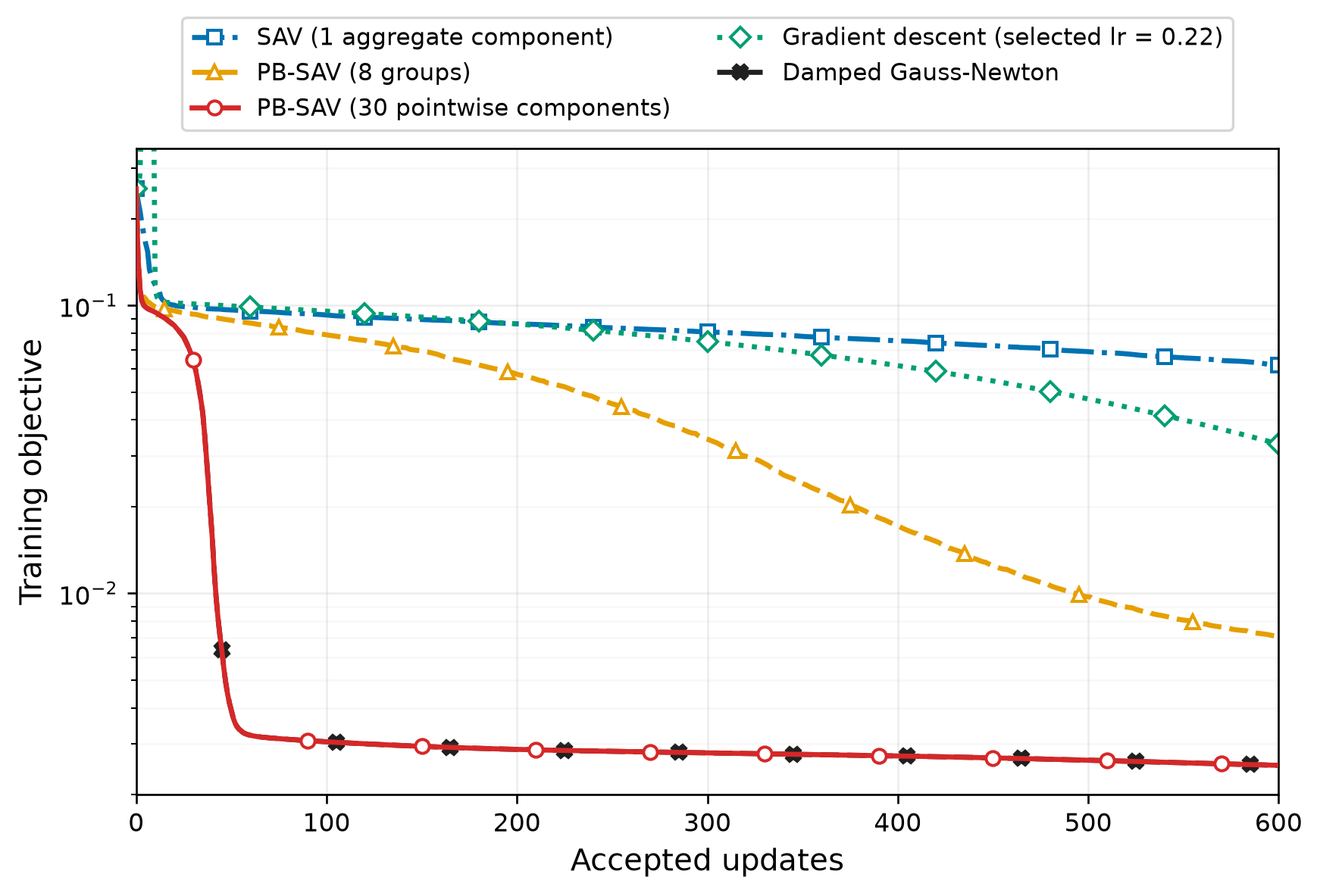}
  \caption{Training objective for one, eight, and 30 PB--SAV components,
  together with gradient descent and damped Gauss--Newton.  The pointwise
  PB--SAV and Gauss--Newton trajectories are indistinguishable at plotting
  resolution.}
  \label{fig:regression-convergence}
\end{figure}

\begin{figure}[t]
  \centering
  \includegraphics[width=\textwidth]{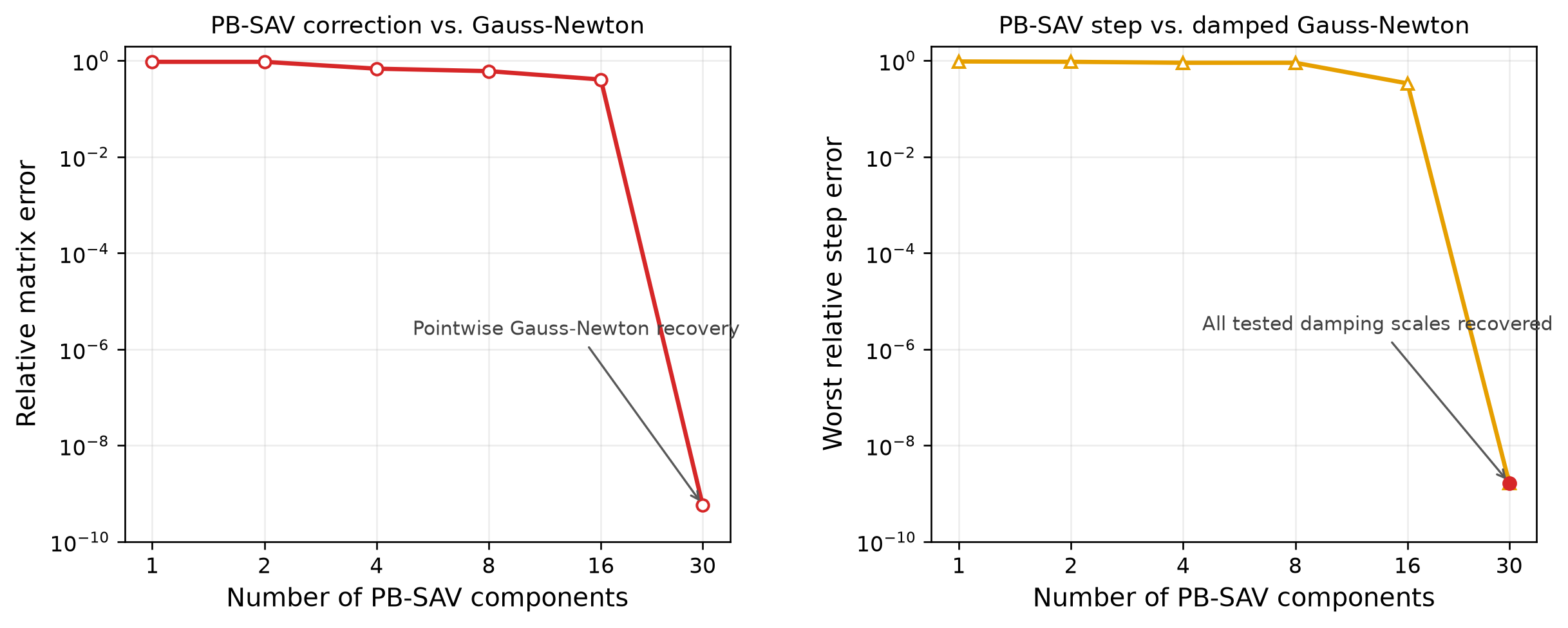}
  \caption{Fixed-state geometry for nonlinear regression.  Left: relative
  error between the PB--SAV correction and the Gauss--Newton matrix.  Right:
  the largest relative step error over \(\eta=0.1,1,10,100\).  Both errors
  collapse when each observation is represented by a separate component.}
  \label{fig:regression-recovery}
\end{figure}

\begin{table}[t]
\centering
\caption{Nonlinear-regression results after 600 accepted updates.  No entry is
bold: PB--SAV and damped Gauss--Newton agree in every column to the displayed
precision.}
\label{tab:regression-results}
\footnotesize
\setlength{\tabcolsep}{3pt}
\begin{tabular}{@{}lcccc@{}}
\toprule
Method & Updates & \shortstack{Final\\objective} &
\shortstack{Best noise-free\\test MSE} & \shortstack{Final\\\(\|\nabla F\|\)} \\
\midrule
PB--SAV (30 pointwise comp.) & 600 & \(2.525\times10^{-3}\) & \(2.804\times10^{-3}\) & \(1.032\times10^{-4}\) \\
Damped Gauss--Newton & 600 & \(2.525\times10^{-3}\) & \(2.804\times10^{-3}\) & \(1.032\times10^{-4}\) \\
Gradient descent (rate 0.22) & 600 & \(3.301\times10^{-2}\) & \(6.040\times10^{-2}\) & \(2.468\times10^{-2}\) \\
Adam & 600 & \(2.985\times10^{-3}\) & \(3.113\times10^{-3}\) & \(6.695\times10^{-4}\) \\
\bottomrule
\end{tabular}
\end{table}

\Cref{fig:regression-convergence} shows the trajectories and
\cref{fig:regression-recovery} the fixed-state errors.
At the common initial state, the pointwise PB--SAV correction differs from the
Gauss--Newton matrix by only \(5.82\times10^{-10}\).  The relative step errors
for \(\eta=0.1,1,10,100\) are \(2.31\times10^{-10}\),
\(3.71\times10^{-10}\), \(9.72\times10^{-10}\), and \(1.67\times10^{-9}\),
and the direction cosine equals one to displayed precision in every case.
The effective correction rank is 17 rather than the nominal 30, consistent
with the rank of the sample Jacobian at the initial state.  Along the full
trajectory, pointwise PB--SAV and damped Gauss--Newton agree in final objective
to \(1.24\times10^{-11}\).  They also attain the same noise-free test MSE and final
gradient norm in \cref{tab:regression-results}.

The pointwise components expose the individual residual Jacobians, so the
PB--SAV correction reproduces the Gauss--Newton geometry up to the small shift
regularization. Coarser partitions aggregate these directions before forming
the pullback.

\subsection{Operator learning: two-dimensional Darcy flow}
\label{sec:darcy-experiment}

We next apply PB--SAV with separate components for the supervised data loss
and weight decay to operator learning and compare it with heavy-ball gradient descent and
AdamW. For each
permeability field \(a(x,y)\), the target solution satisfies
\begin{equation}
  -\nabla\!\cdot\!\left(a(x,y)\nabla u(x,y)\right)=1,
  \qquad (x,y)\in\Omega=(0,1)^2,
  \qquad u|_{\partial\Omega}=0.
  \label{eq:darcy-problem}
\end{equation}
The permeability is generated by thresholding a smoothed Gaussian random
field with correlation length \(0.12\), which produces the two values 3 and
12.  Finite-difference solutions are represented on a \(33\times33\) grid.
For each seed, the dataset contains 128 training functions and 64 independently
generated test functions.

We use a 51,521-parameter convolutional-branch deep operator network
(DeepONet) \cite{lu2021deeponet}. The branch encoder uses convolutional channel widths
\(8,16,32\), adaptive \(4\times4\) pooling, and a 64-dimensional projection.
The trunk is a three-hidden-layer tanh network of width 64 and maps spatial
coordinates to the same latent dimension.  Multiplication by
\(16x(1-x)y(1-y)\) enforces the homogeneous Dirichlet boundary condition exactly.

Let \(N_{\rm tr}=128\) and \(N_x=33^2\).  For training field \(a_j\), let
\(u_j(x_\ell)\) be the reference solution at grid point \(x_\ell\), and let
\(\widehat u_\theta(a_j;x_\ell)\) be the network prediction.  Define
\begin{align}
  E_{\rm data}(\theta)
  &=\frac{1}{2N_{\rm tr}N_x}
    \sum_{j=1}^{N_{\rm tr}}
    \sum_{\ell=1}^{N_x}
    \left(\widehat u_\theta(a_j;x_\ell)-u_j(x_\ell)\right)^2,\\
  E_{\rm wd}(\theta)
  &=\frac{\lambda}{2}\norm{\theta}^2,
  \qquad \lambda=3\times10^{-5},
\end{align}
PB--SAV uses the two-component objective
\begin{equation}
  F_{\rm Darcy}(\theta)=E_{\rm data}(\theta)+E_{\rm wd}(\theta).
  \label{eq:darcy-objective}
\end{equation}
The common training quantity reported here is \(E_{\rm data}\).
Generalization is measured by the relative \(L^2\) error on the held-out
operator dataset.

The selected learning rates are \(10^{-2}\), \(10^{-3}\), and
\(3\times10^{-3}\) for heavy-ball gradient descent, AdamW, and PB--SAV,
respectively. PB--SAV uses \(\beta=0.9\),
\(\beta_2=0.999\), \(\epsilon=10^{-8}\), \(\alpha=0.5\), and
\(\rho^n\equiv1\). These PB--SAV parameters are also used in the following Burgers experiments unless otherwise stated.  Each
method runs for 5,000 full-batch updates using the paired seeds
\(42, 101, 211, 307, 401\).
The displayed field is the fixed test sample with the greatest variation in
the reference solution.

\begin{figure}[t]
  \centering
  \includegraphics[width=\textwidth]{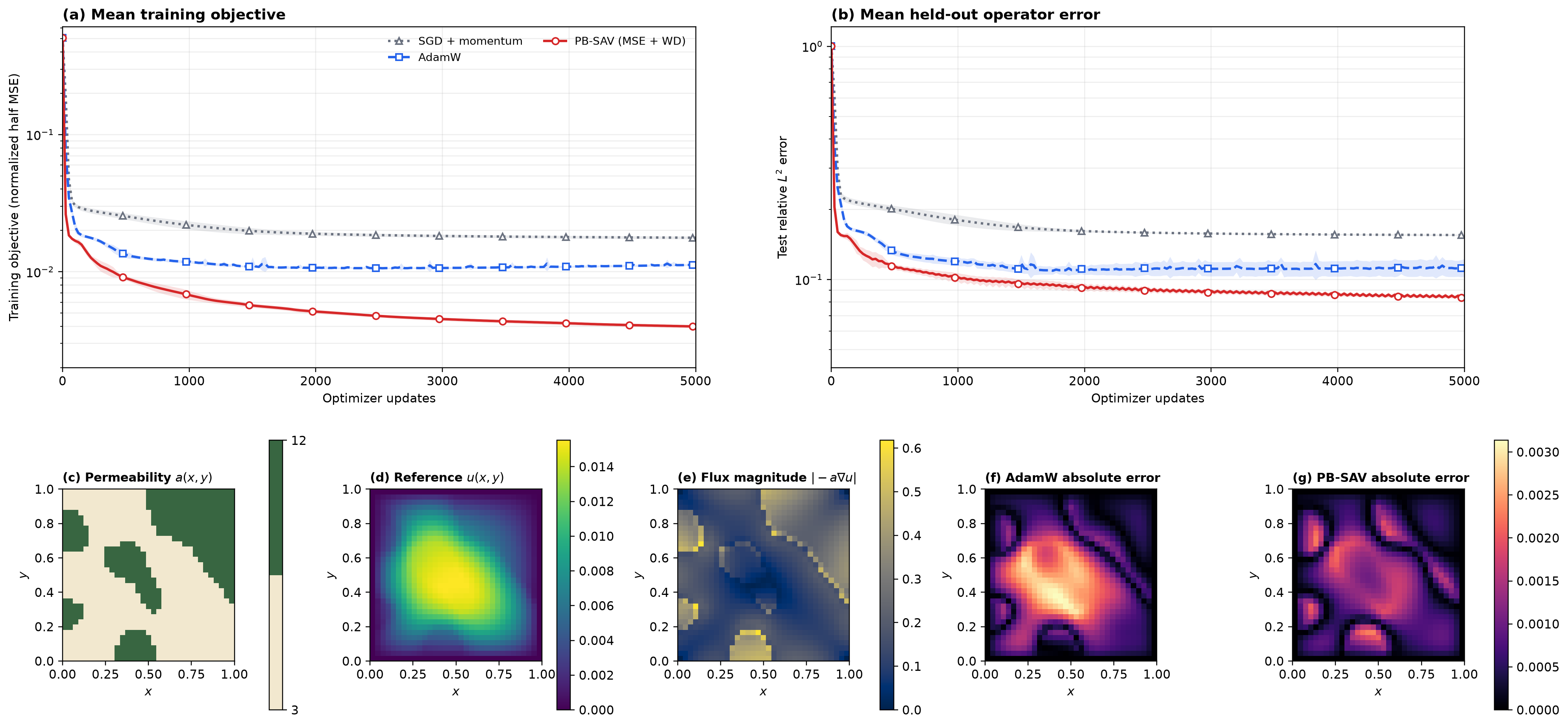}
  \caption{Five-seed Darcy DeepONet comparison.  (a) Mean common training data
  loss and (b) mean held-out relative \(L^2\) error; shaded regions show one
  sample standard deviation.  The legend entry ``SGD + momentum'' in (a) and
  (b) is the full-batch heavy-ball gradient descent of \cref{tab:darcy-results}.  The lower row shows a fixed test sample selected
  by the maximum variation of the reference field: (c) permeability, (d) reference
  solution, (e) flux magnitude, (f) AdamW absolute error, and (g) PB--SAV
  absolute error.}
  \label{fig:darcy-results}
\end{figure}

\begin{table}[t]
\centering
\caption{Darcy DeepONet results after 5,000 updates over five paired seeds.
Losses and errors are reported as mean \(\pm\) sample standard deviation.  Wall
time is the mean recorded time for the full training run.  Best mean training
and test values are bold; wall time is reported for reference and is not
ranked.}
\label{tab:darcy-results}
\scriptsize
\setlength{\tabcolsep}{3pt}
\begin{tabular}{@{}lccccc@{}}
\toprule
Method & Components & \(\eta\) & \shortstack{Train data\\loss} &
\shortstack{Test relative\\\(L^2\)} & \shortstack{Wall time\\(s)} \\
\midrule
\shortstack[l]{Heavy-ball\\gradient descent} & -- & \(10^{-2}\) & \(0.01771\pm0.00041\) & \(0.15510\pm0.00095\) & 37.4 \\
AdamW & -- & \(10^{-3}\) & \(0.01125\pm0.00024\) & \(0.11264\pm0.01001\) & 36.5 \\
\shortstack[l]{PB--SAV (data +\\weight decay)} & 2 & \(3\times10^{-3}\) & \(\mathbf{0.00399\pm0.00011}\) & \(\mathbf{0.08512\pm0.00225}\) & 74.5 \\
\bottomrule
\end{tabular}
\end{table}

PB--SAV attains the lowest training and test errors in
\cref{fig:darcy-results,tab:darcy-results}.  Relative to AdamW, it reduces the
mean training data loss by 64.5\% and the mean held-out error by 24.4\%.
Relative to heavy-ball gradient descent, the reductions are 77.5\% and
45.1\%, respectively.
The standard deviation of the PB--SAV test error is \(0.00225\), compared with
\(0.01001\) for AdamW.  In the lower row of \cref{fig:darcy-results}, the PB--SAV prediction has a
smaller error over most of the domain for the displayed test sample.

With two components, PB--SAV gives lower mean training and test errors than
the two baselines at approximately twice the recorded wall time of AdamW.
This comparison evaluates the complete optimizer. The following forward
Burgers experiment includes a one-component PB--SAV baseline to examine the
effect of the decomposition itself.

\subsection{Forward Burgers PINN: one versus four loss components}
\label{sec:burgers-forward-experiment}

The forward problem compares one aggregated component with four resolved
components at \(\alpha=1/2\), under the same momentum and adaptive scaling
rules. The inverse problem then adds sparse observations and an unknown
viscosity.

We consider the viscous Burgers equation
\begin{equation}
  u_t+uu_x-\nu u_{xx}=0,
  \qquad (x,t)\in[-1,1]\times[0,1],
  \qquad \nu=0.01,
  \label{eq:burgers-forward-pde}
\end{equation}
with initial and boundary conditions
\begin{equation}
  u(x,0)=-\sin(\pi x),
  \qquad u(-1,t)=u(1,t)=0.
  \label{eq:burgers-forward-data}
\end{equation}
The PINN is a tanh network with six hidden layers of width 64 and 21,057
trainable parameters. Every run uses 10,000 fixed scrambled Sobol collocation
points \cite{sobol1967,owen1997}, 256 initial points, 256 time points on each
boundary, and 10,000
independent validation collocation points.  The relative solution error is
evaluated against a grid-refined backward differentiation formula (BDF)
method-of-lines reference on a
\(256\times101\) space--time grid.

Let \(r_f=u_t+uu_x-\nu u_{xx}\).
For sampled values \(w_1,\ldots,w_N\), we use
\(\operatorname{MSE}(w)=N^{-1}\sum_{j=1}^N w_j^2\).
The common PINN task objective is
\begin{subequations}
\label{eq:burgers-forward-losses}
\begin{align}
  E_{\rm res}
  &=\frac12\operatorname{MSE}(r_f),\\
  E_{\rm bc}
  &=\frac14\left[
    \operatorname{MSE}(u(-1,t))+
    \operatorname{MSE}(u(1,t))\right],\\
  E_{\rm ic}
  &=\frac12\operatorname{MSE}(u(x,0)+\sin(\pi x)),\\
  F_{\rm task}
  &=E_{\rm res}+E_{\rm bc}+E_{\rm ic}.
\end{align}
\end{subequations}
All methods use the weight-decay coefficient \(\lambda=10^{-6}\).  For
PB--SAV, weight decay is represented by the explicit component
\begin{equation}
  F_{\rm reg}(\theta)
  =F_{\rm task}(\theta)+E_{\rm wd}(\theta),
  \qquad
  E_{\rm wd}(\theta)=\frac{\lambda}{2}\norm{\theta}^2.
  \label{eq:burgers-forward-regularized}
\end{equation}
We report \(F_{\rm task}\), which is common to all methods, as the primary
optimization quantity.

We compare heavy-ball gradient descent, AdamW, adaptive PB--SAV with momentum
using one
aggregate component, and the same PB--SAV method using the four components
\(E_{\rm res}\), \(E_{\rm bc}\), \(E_{\rm ic}\), and \(E_{\rm wd}\).
The total SAV shift is \(10^{-12}\).  One quarter is assigned to weight decay,
and the remainder is distributed among the three physical terms in proportion
to their sample counts.
A seed-0 pilot selects learning rates from the task objective only.  The
selected rates are \(10^{-2}\) for heavy-ball gradient descent and
\(10^{-3}\) for AdamW and both PB--SAV variants.  Reported runs use 10,000
updates at the paired seeds \(42, 2273, 2669, 5439, 8041\).

\begin{figure}[t]
  \centering
  \includegraphics[width=\textwidth]{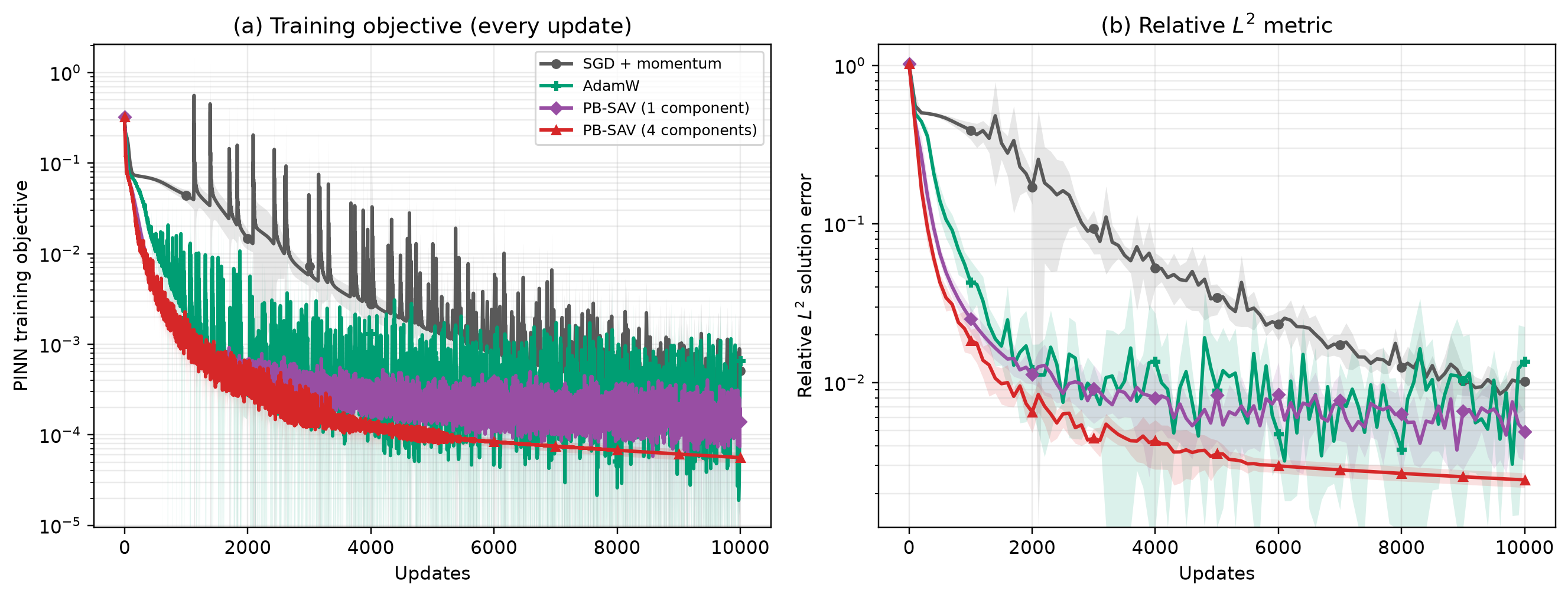}
  \caption{Five-seed comparison for the forward Burgers PINN.  (a) Mean common
  task objective, recorded at every update, and (b) mean relative \(L^2\)
  solution error.  Shaded regions show one sample standard deviation.}
  \label{fig:burgers-forward-results}
\end{figure}

\begin{table}[t]
\centering
\caption{Forward Burgers results after 10,000 updates over five paired seeds.
For each seed, the tail mean and coefficient of variation (CV) are computed
over the final 2,000 updates; the table reports their mean \(\pm\) sample
standard deviation across seeds.  Here
\(\mathrm{CV}=100s_{\rm tail}/\bar F_{\rm tail}\).  Best mean values are
bold.}
\label{tab:burgers-forward-results}
\footnotesize
\setlength{\tabcolsep}{3pt}
\begin{tabular}{lccc}
\toprule
Method & \(\eta\) & \shortstack{Tail task\\objective} &
\shortstack{Tail objective\\CV} \\
\midrule
Heavy-ball GD & 0.01 & \((5.095\pm0.660)\times10^{-4}\) & \(130.2\%\pm22.0\%\) \\
AdamW & 0.001 & \((3.181\pm0.450)\times10^{-4}\) & \(161.5\%\pm22.0\%\) \\
PB--SAV (1 component) & 0.001 & \((1.741\pm0.430)\times10^{-4}\) & \(58.9\%\pm31.0\%\) \\
PB--SAV (4 components) & 0.001 & \(\mathbf{(6.154\pm0.840)\times10^{-5}}\) & \(\mathbf{5.4\%\pm0.3\%}\) \\
\bottomrule
\end{tabular}

\medskip
\begin{tabular}{lc}
\toprule
Method & Final relative \(L^2\) error \\
\midrule
Heavy-ball GD & \((1.018\pm0.330)\times10^{-2}\) \\
AdamW & \((1.356\pm0.890)\times10^{-2}\) \\
PB--SAV (1 component) & \((4.893\pm1.700)\times10^{-3}\) \\
PB--SAV (4 components) & \(\mathbf{(2.435\pm0.260)\times10^{-3}}\) \\
\bottomrule
\end{tabular}
\end{table}

The mean wall times were \(316.6\pm8.4\), \(328.8\pm6.1\),
\(491.8\pm6.5\), and \(496.9\pm16.0\) seconds for heavy-ball gradient
descent, AdamW, one-component PB--SAV, and four-component PB--SAV,
respectively.

The four-component PB--SAV method attains the lowest task objective, the
lowest solution error, and the smallest late-stage oscillation in
\cref{fig:burgers-forward-results,tab:burgers-forward-results}.  Relative to
one-component PB--SAV, it reduces the mean tail objective by 64.7\%, the mean
final solution error by 50.2\%, and the mean tail CV from 58.9\% to 5.4\%.
Relative to AdamW, the reductions in task objective and solution error are
80.7\% and 82.0\%, respectively.  Relative to heavy-ball gradient descent,
they are 87.9\%
and 76.1\%.

Splitting the objective into four components costs 1.0\% of the recorded
wall time over the one-component run, once the common PINN forward graph has
been evaluated; the one- and four-component PB--SAV runs are respectively
49.6\% and 51.1\% slower than AdamW.

Both runs form the same component gradients.  The one-component run supplies a
single aggregated component, so its curvature gap vanishes and its correction
is the aggregate \(B_1^n\) for every \(\alpha\); the four-component run has
\(m=4\) and \(\alpha=1/2\), so \cref{eq:mixed-pullback} gives
\(B_\alpha^n=\tfrac12(B_1^n+B_m^n)\).  The two runs use the same architecture,
samples, mobility update rule, momentum parameters, and learning rate.
Keeping the residual, boundary, initial, and regularization directions
separate produces a smoother late optimization trajectory and a more accurate
PDE solution in these runs.

\subsection{Inverse Burgers PINN: viscosity identification}
\label{sec:burgers-inverse-experiment}

The final experiment uses the same Burgers equation and prescribed conditions,
but treats the viscosity as unknown.  The reference value is
\(\nu_*=0.01\), and optimization starts from \(\nu^0=0.02\).  We optimize
\(\xi=\log\nu\), which enforces \(\nu>0\).  Let \(\theta_{\rm net}\in\R^{21057}\) denote the
weights of the PINN with six hidden layers of width 64.  The full optimization
state is \(z=(\theta_{\rm net},\xi)\in\R^{21058}\), with \(\nu=e^\xi\).

Every run uses 10,000 fixed collocation points, 256 initial points, 256 points
on each boundary, and 100 noiseless interior observations sampled from the
validated reference solution.  An independent set of 10,000 collocation
points is used for validation.  The relative solution error is evaluated on
the same \(256\times101\) grid used in the forward problem.

The residual, boundary, and initial energies are defined as in
\cref{eq:burgers-forward-losses}, with the trainable viscosity inserted into
the residual.  Let \((x_j^{\rm obs},t_j^{\rm obs})\) be the observation sites
and \(y_j\) the corresponding reference values.  The observation loss is
\begin{equation}
  E_{\rm data}
  =\frac{1}{200}\sum_{j=1}^{100}
  \left(u_{\theta_{\rm net}}(x_j^{\rm obs},t_j^{\rm obs})-y_j\right)^2.
  \label{eq:burgers-data-loss}
\end{equation}
The common inverse-task objective and the PB--SAV regularized objective are
\begin{equation}
  F_{\rm inv}
  =E_{\rm res}+E_{\rm bc}+E_{\rm ic}+E_{\rm data},
  \qquad
  F_{\rm inv}^{\lambda}
  =F_{\rm inv}+\frac{\lambda}{2}\norm{\theta_{\rm net}}^2.
  \label{eq:burgers-inverse-objective}
\end{equation}
The regularizer acts only on the network parameters.  The viscosity parameter
is identified through the physics and observation terms.

The one-component PB--SAV method aggregates all task and regularization terms.
The four-component formulation uses \(E_{\rm res}\),
\(E_{\rm bc}+E_{\rm ic}\), \(E_{\rm data}\), and \(E_{\rm wd}\).
The boundary and initial terms are grouped as prescribed-condition losses,
while the residual, observation, and regularization directions remain
separate.  Component shifts are additive under aggregation.  Of the total
shift \(10^{-12}\), 25\% is assigned to regularization.

The pilot described below selects the weight-decay coefficient
\(\lambda=3\times10^{-7}\).  We report \(F_{\rm inv}\), which is common to all
methods.

A separate 500-update pilot on seed 0 selects learning rates using the tail
training objective with a penalty for late objective increases.  The selected
rates are \(10^{-2}\), \(10^{-3}\), \(3\times10^{-3}\), and \(10^{-2}\) for
heavy-ball gradient descent, AdamW, one-component PB--SAV, and four-component
PB--SAV, respectively.  Reported runs use 10,000 updates at the same five seeds as the
forward experiment.

\begin{figure}[t]
  \centering
  \includegraphics[width=\textwidth]{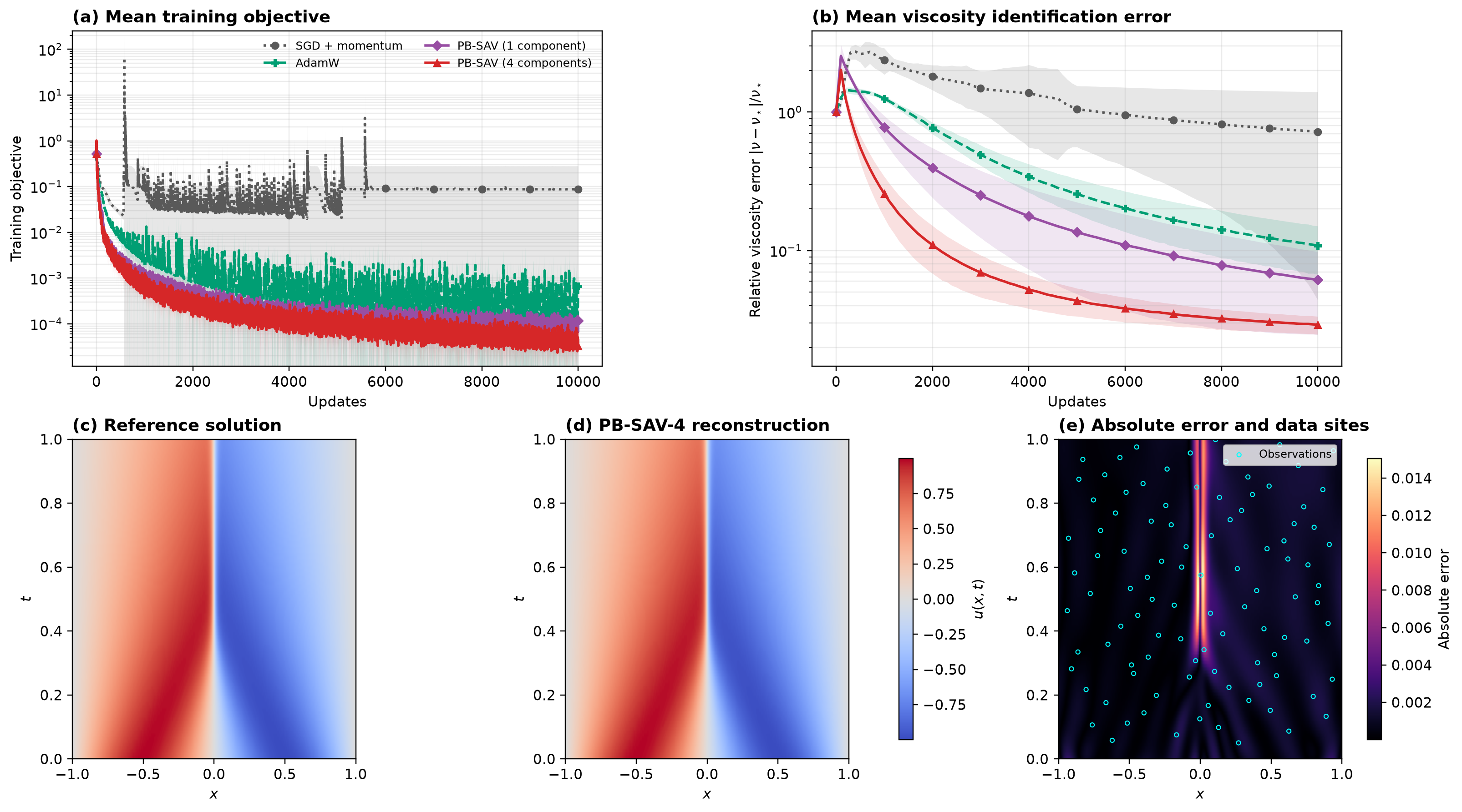}
  \caption{Five-seed comparison for inverse Burgers identification.  (a) Mean
  common task objective and (b) mean relative viscosity error; shaded regions
  show one sample standard deviation.  The lower row uses seed 42, the median
  four-component PB--SAV seed according to final viscosity error: (c) reference
  solution, (d) four-component PB--SAV reconstruction, and (e) absolute error
  with the 100 observation
  locations.}
  \label{fig:burgers-inverse-results}
\end{figure}

\begin{table}[t]
\centering
\caption{Inverse Burgers results after 10,000 updates over five paired seeds.
Values are mean \(\pm\) sample standard deviation.  For each seed, the tail
task objective is averaged over the final 2,000 updates of \(F_{\rm inv}\).
Best mean values are bold.}
\label{tab:burgers-inverse-results}
\footnotesize
\setlength{\tabcolsep}{3pt}
\begin{tabular}{lcc}
\toprule
Method & \(\eta\) & \shortstack{Tail task\\objective} \\
\midrule
Heavy-ball GD & 0.01 & \((0.878\pm1.900)\times10^{-1}\) \\
AdamW & 0.001 & \((3.515\pm0.330)\times10^{-4}\) \\
PB--SAV (1 component) & 0.003 & \((9.507\pm2.000)\times10^{-5}\) \\
PB--SAV (4 components) & 0.01 & \(\mathbf{(4.816\pm2.000)\times10^{-5}}\) \\
\bottomrule
\end{tabular}

\medskip
\begin{tabular}{lcc}
\toprule
Method & \shortstack{Final viscosity\\relative error} &
\shortstack{Final relative\\\(L^2\) error} \\
\midrule
Heavy-ball GD & \((7.213\pm6.800)\times10^{-1}\) & \((2.293\pm4.300)\times10^{-1}\) \\
AdamW & \((1.088\pm0.420)\times10^{-1}\) & \((1.440\pm0.650)\times10^{-2}\) \\
PB--SAV (1 component) & \((6.141\pm3.700)\times10^{-2}\) & \((7.476\pm3.300)\times10^{-3}\) \\
PB--SAV (4 components) & \(\mathbf{(2.914\pm0.430)\times10^{-2}}\) & \(\mathbf{(3.356\pm0.310)\times10^{-3}}\) \\
\bottomrule
\end{tabular}
\end{table}

The mean wall times were \(404.8\pm9.2\), \(408.2\pm8.2\),
\(567.1\pm33.0\), and \(583.6\pm35.0\) seconds for heavy-ball gradient
descent, AdamW, one-component PB--SAV, and four-component PB--SAV,
respectively.  All four methods use \(\lambda=3\times10^{-7}\).

Four-component PB--SAV gives the best mean result for the training objective,
viscosity identification, and solution reconstruction in
\cref{fig:burgers-inverse-results,tab:burgers-inverse-results}.  Relative to
one-component PB--SAV at its own selected learning rate, it reduces the mean
tail objective by 49.3\%, the mean viscosity error by 52.5\%, and the mean
solution error by 55.1\%.  Its solution-error standard deviation is
approximately 10.6 times smaller, and its viscosity-error standard deviation
is approximately 8.6 times smaller.
Relative to AdamW, the mean reductions in these three quantities are 86.3\%,
73.2\%, and 76.7\%, respectively.  The four-component method is 2.9\% slower
than one-component PB--SAV in the recorded wall time and 43.0\% slower than
AdamW.

Heavy-ball gradient descent converges to a near-zero field on seed 2669,
which accounts for the large mean and variance in its aggregate row.  AdamW
approaches the viscosity on all five seeds, but its task objective exhibits
large late spikes.

The representative reconstruction in \cref{fig:burgers-inverse-results}
tracks the reference solution throughout most of the space--time domain.  The
largest remaining error occurs near the steep central layer, where the
solution is most sensitive to viscosity. The four-component optimizer attains
lower mean errors at its selected learning rate with a small additional cost
over one-component PB--SAV; the forward experiment compares the pullbacks at a
common learning rate.

\FloatBarrier

\section{Conclusions}
\label{sec:conclusion}

We placed the PB--SAV component correction inside an optimizer with momentum
and an adaptive mobility, applying it to the gradient and the stored momentum
in a single implicit solve. The update dissipates an exact modified energy for
every mobility that is nonincreasing in the Loewner order, and its local
stability at a stationary point is governed by the Hessian minus twice the
correction. The aggregate correction vanishes at such a point while the
component correction need not, so resolving the objective into components can
extend the admissible step sizes. Each iteration adds one dense solve of order
at most the number of components.

The quadratic and regression tests show that the correction can reproduce a
reference increment while its error as a matrix remains large. In the forward
Burgers comparison, four components reduce the mean tail objective by 64.7\%
and the final solution error by 50.2\% relative to one component at the same
learning rate and momentum settings. How to choose the components when the
objective is not already split, and how to do so without one gradient
evaluation per component, remain open.

\section*{Declaration of competing interest}
The authors declare no known competing financial interests or personal
relationships that could have appeared to influence the work reported in this
paper.

\section*{Data and code availability}
Code and data will be made available in a public repository on publication.

\section*{Acknowledgements}
We gratefully acknowledge support from the National Science Foundation (DMS-2533878, DMS-2053746, DMS-2134209, ECCS-2328241, CBET-2347401 and OAC-2311848), and U.S.~Department of Energy (DOE) Office of Science Advanced Scientific Computing Research program under the "Uncertainty Quantification for Multifidelity Operator Learning (MOLUcQ)" project (Project No. 81739), DE-SC0023161, the SciDAC LEADS Institute, and DOE–Fusion Energy Science, under grant number: DE-SC0024583.

\appendix
\renewcommand{\thesubsection}{\Alph{section}.\arabic{subsection}}
\renewcommand{\thetheorem}{\Alph{section}.\arabic{theorem}}
\aliascntresetthe{proposition}
\aliascntresetthe{corollary}
\section{Convergence proofs}
\label[appendix]{app:convergence-proofs}

\subsection{Proof of the modified energy law}
\label[appendix]{app:energy-law-proof}

\begin{proof}[Proof of \cref{thm:monotone-mobility-dissipation}]
Take the inner product of \cref{eq:abstract-momentum-balance} with
\(\Delta\theta^{n+1}=M^{n+1}p^{n+1}\).  Weighted polarization gives
\begin{align}
&\frac{(p^{n+1})^{\mathsf T}M^{n+1}
  (p^{n+1}-\beta p^n)}{\eta}
\nonumber\\
&\quad=\frac{1}{2\eta}\left(
  \norm{p^{n+1}}_{M^{n+1}}^2
  -\beta^2\norm{p^n}_{M^{n+1}}^2
  +\norm{p^{n+1}-\beta p^n}_{M^{n+1}}^2\right).
  \label{eq:weighted-polarization}
\end{align}
Insert and subtract \(\norm{p^n}_{M^n}^2/(2\eta)\) to obtain
\begin{align*}
&\frac{\norm{p^{n+1}}_{M^{n+1}}^2
  -\beta^2\norm{p^n}_{M^{n+1}}^2}{2\eta}\\
&\quad=
\frac{\norm{p^{n+1}}_{M^{n+1}}^2-\norm{p^n}_{M^n}^2}{2\eta}
+\frac{\norm{p^n}_{M^n-M^{n+1}}^2}{2\eta}
+\frac{1-\beta^2}{2\eta}\norm{p^n}_{M^{n+1}}^2.
\end{align*}
Substitute the scalar and pullback identities in
\cref{eq:common-identities}.  Collecting the scalar and kinetic energy
changes yields \cref{eq:monotone-mobility-dissipation}.  Each remaining
term is nonnegative because \(S_m^n\succeq0\),
\(M^n-M^{n+1}\succeq0\), and \(0\leq\beta<1\).
\end{proof}

\subsection{Proof of the local stability criterion}
\label[appendix]{app:schur-proof}

\begin{proof}[Proof of \cref{thm:local-schur-stability}]
Write \(\widehat H_*=M^{1/2}H_*M^{1/2}\) and
\(\widehat B_*=M^{1/2}B_*M^{1/2}\). For the linearized variables
\(x^n=M^{-1/2}(\theta^n-\theta_*)\) and \(w^n=M^{1/2}p^n\),
\begin{equation}
  (I_d+\eta\widehat B_*)w^{n+1}
  =\beta w^n-\eta\kappa_*\widehat H_*x^n,
  \qquad x^{n+1}=x^n+w^{n+1}.
  \label{eq:linearized-momentum-map}
\end{equation}
Derivatives of the inverse coefficient multiply a zero right-hand side,
and variations of \(q/Q\) multiply the zero gradient. The invertible change
of state from \((x^n,w^n)\) to \((x^n,x^n-w^n)\) gives the matrix polynomial
\begin{equation}
\begin{aligned}
  \mathcal P(z)
  ={}&z^2(I_d+\eta\widehat B_*)\\
  &-z\bigl((1+\beta)I_d+\eta\widehat B_*
                     -\eta\kappa_*\widehat H_*\bigr)+\beta I_d.
\end{aligned}
\label{eq:local-characteristic-polynomial}
\end{equation}
Its determinant vanishes exactly at the eigenvalues of the state map,
including when \(\beta=0\).

If \(\mathcal P(z)\xi=0\) with \(\xi\in\mathbb C^d\setminus\{0\}\),
multiplication by \(\xi^*\) gives the scalar quadratic
\begin{equation}
  (1+\eta b)z^2
  -(1+\beta+\eta b-\eta\kappa_*h)z+\beta=0,
  \label{eq:local-rayleigh-polynomial}
\end{equation}
where \(b=\xi^*\widehat B_*\xi/(\xi^*\xi)\geq0\) and
\(h=\xi^*\widehat H_*\xi/(\xi^*\xi)>0\) are real. The leading
coefficient exceeds \(\beta\). The values of this polynomial at \(1\) and
\(-1\) are \(\eta\kappa_*h\) and
\(2(1+\beta)+2\eta b-\eta\kappa_*h\), respectively, both positive under
\cref{eq:local-stability-criterion}. Nonreal roots have squared modulus
\(\beta/(1+\eta b)<1\). For real roots, the same root product and the
positive values at \(1\) and \(-1\) exclude roots outside \((-1,1)\).
This proves sufficiency.

Conversely,
\[
  \mathcal P(-1)=2(1+\beta)I_d+2\eta\widehat B_*
                                 -\eta\kappa_*\widehat H_*.
\]
If this matrix is singular, \(-1\) is an eigenvalue. If it has a negative
eigenvalue, \(\mathcal P(-t)\) becomes positive definite for sufficiently
large \(t\), since its leading coefficient is positive definite.
Continuity of its smallest eigenvalue gives a singular matrix
\(\mathcal P(-t)\) at some real \(t>1\), and hence a state eigenvalue \(-t<-1\).
This proves necessity.
\end{proof}

\subsection{Comparison with the appended recurrence}
\label[appendix]{app:appended-comparison}

On a common invariant one-dimensional mode of \(\widehat H_*\) and
\(\widehat B_*\), let \(h>0\) and \(b\geq0\) be their curvatures.
If \(b=\kappa_*h\), the characteristic equation reduces to
\[
  (1+\eta b)z^2-(1+\beta)z+\beta=0.
\]
For \(\beta>0\) and \(\eta b>(1-\beta)^2/(4\beta)\), its roots have modulus
\(\sqrt{\beta/(1+\eta b)}\), which tends to zero as \(\eta b\to\infty\).
For \(\beta=0\), the roots are \(0\) and \(1/(1+\eta b)\).
The frozen appended recurrence has root product \(\beta\), so its spectral
radius is at least \(\sqrt\beta\). On modes with large \(\eta b\) the
momentum update therefore has the smaller spectral radius; this is a
stiff-limit statement, not a uniform rate comparison. Its stability
condition, however, is
\[
  \eta\kappa_*H_*
  \prec2(1+\beta)(M^{-1}+\eta B_*),
\]
which admits at least the interval of the momentum update for the same
pullback. It follows by diagonalizing the positive definite generalized
Hessian pair \((H_*,M^{-1}+\eta B_*)\). The interval enlargement in
\cref{cor:local-step-interval} compares component and aggregate pullbacks
within the momentum update.

\subsection{Local convergence with scalar relaxation}
\label[appendix]{app:local-convergence-proof}

\begin{proof}[Proof of \cref{thm:local-relaxed-convergence}]
Set \(z=(\theta-\theta_*,p)\) and \(Q_*=Q(\theta_*)\). Let
\(T(q,z)\) be the parameter--momentum update, without applying scalar
relaxation. This map is continuously differentiable near \((Q_*,0)\), and
\(T(q,0)=0\) for every nearby \(q\). Its derivative
\(L=D_zT(Q_*,0)\) is Schur stable by
\cref{thm:local-schur-stability}. The convergent series
\[
  P=\sum_{j=0}^{\infty}(L^{\mathsf T})^jL^j
\]
defines a positive definite matrix with
\(L^{\mathsf T}PL=P-I\). Hence \(\norm{L}_P<1\).
Continuity of \(D_zT\) and integration along the segment from \(0\) to
\(z\) give a cylinder
\[
  \norm{z}_P\leq a_z,\qquad |q-Q_*|\leq a_q<Q_*/2
\]
and \(r\in(0,1)\) such that
\begin{equation}
  \norm{T(q,z)}_P\leq r\norm{z}_P.
  \label{eq:local-uniform-contraction}
\end{equation}

The scalar must remain in this cylinder. Initialization and the acceptance
rule give the invariant \(0\leq q^n\leq Q(\theta^n)\). Within the cylinder,
uniformly in \(\rho\in[0,1]\), the increment of
\cref{eq:abstract-momentum-step} and the scalar increment \(d\) of
\cref{eq:common-d} satisfy
\[
  \norm{\Delta}+\norm{p^+}=O(\norm{z}_P),\qquad
  d=O(\norm{z}_P^2),\qquad
  0\leq\mathcal D=O(\norm{z}_P^2).
\]
Write the two branches of \cref{eq:generic-relaxation} as
\(q_{\mathrm{bud}}^2=(q+d)^2+\rho\mathcal D\) and
\(q_{\mathrm{cap}}^2=Q(\theta+\Delta)^2\). Then
\[
  |q_{\mathrm{bud}}^2-q^2|\leq c_4\norm{z}_P^2.
\]
Since \(q\leq Q(\theta)\) and \(\nabla F(\theta_*)=0\),
\[
  q_{\mathrm{cap}}^2-q^2\geq F(\theta+\Delta)-F(\theta)
  \geq-c_5\norm{z}_P^2.
\]
Both branches are bounded below by \(q^2-c_6\norm{z}_P^2\) with
\(c_6=\max\{c_4,c_5\}\), while the budget branch bounds their minimum
above. Thus
\[
  |(q^+)^2-q^2|\leq c_6\norm{z}_P^2.
\]
Dividing by \(q^++q\geq q\geq Q_*/2\) gives
\begin{equation}
  |q^+-q|\leq c_2\norm{z}_P^2.
  \label{eq:local-scalar-drift}
\end{equation}
This estimate covers either selected branch and branch ties.

Because \(\nabla Q(\theta_*)=0\),
\(|q^0-Q_*|\leq c_3\norm{z^0}_P^2\). Choose the initial neighborhood
so that \(\norm{z^0}_P<a_z\) and
\[
  \left(c_3+\frac{c_2}{1-r^2}\right)\norm{z^0}_P^2<a_q.
\]
Induction using \cref{eq:local-uniform-contraction,eq:local-scalar-drift}
then gives
\[
  \norm{z^n}_P\leq r^n\norm{z^0}_P,
  \qquad
  |q^n-Q_*|
  \leq\left(c_3+\frac{c_2}{1-r^2}\right)\norm{z^0}_P^2<a_q.
\]
The iterates cannot leave the cylinder, closing the induction. Scalar
increments are summable, so \(q^n\to q_\infty\geq Q_*/2\). The invariant
\(q^n\leq Q(\theta^n)\) gives \(q_\infty\leq Q_*\), and the drift
bound gives \cref{eq:local-scalar-limit}. Equivalence of norms yields
\cref{eq:local-geometric-convergence}.
\end{proof}

\subsection{Global stationarity estimates}
\label[appendix]{app:global-stationarity-proofs}

With \(L_B=\sup_n\norm{B_\alpha^n}\), the constants of
\cref{cor:best-iterate-stationarity} are
\begin{equation}
\begin{aligned}
  C_{\mathrm{mom}}
  &=\frac{4\eta}{\underline\kappa^2\mu_{\min}}
  \left[
    \left(\eta^{-1}+\mu_{\max}L_B\right)^2
    +\frac{\beta^2\mu_{\max}^2L_B^2}{1-\beta^2}
  \right],\\
  C_{\mathrm{dir}}
  &=\frac{\eta\left(\eta^{-1}+L_B\right)^2}{\underline\kappa^2}.
\end{aligned}
  \label{eq:stationarity-constants-families}
\end{equation}
One sets \(K=C_{\mathrm{mom}}\) under the assumptions of
\cref{thm:conditional-stationarity} and \(K=C_{\mathrm{dir}}\) under
those of \cref{cor:direct-stationarity}.

\begin{proof}[Proof of \cref{thm:conditional-stationarity}]
Sum \cref{eq:relaxed-energy-bound} and use
\cref{eq:strict-relaxation-budget} and \(\Hcal^n\geq0\) to obtain
\(\sum_n\mathcal D^n<\infty\).  The damping and inertial residual terms in
\cref{eq:monotone-mobility-dissipation}, together with the uniform lower
mobility bound, imply
\(
p^n\to0
\)
and
\(
p^{n+1}-\beta p^n\to0
\).
The uniform upper bound gives
\(
\Delta\theta^{n+1}=M^{n+1}p^{n+1}\to0
\).
Using these limits and the boundedness of \(B_\alpha^n\) in
\cref{eq:abstract-momentum-balance} yields
\(
(q^n/Q^n)g^n\to0
\), and the ratio lower bound gives \(g^n\to0\). Since \(F\in C^1\), every
accumulation point is stationary.
\end{proof}

\begin{proof}[Proof of \cref{cor:direct-stationarity}]
The direct update law \cref{eq:base-dissipation} and summability give
\(\Delta\theta^{n+1}\to0\).  The state
equation \cref{eq:base-pbsav-step}, the bounded pullback, and the ratio lower
bound then imply \(g^n\to0\).  The accumulation-point claim follows by
continuity of \(\nabla F\).
\end{proof}

\begin{proof}[Proof of \cref{cor:best-iterate-stationarity}]
For the momentum update, set \(e^n=p^{n+1}-\beta p^n\).  The state
equation, \(p^{n+1}=e^n+\beta p^n\), and the residual and damping terms in
\cref{eq:monotone-mobility-dissipation} give
\begin{align}
  \frac{q^n}{Q^n}g^n
  &=-\frac{e^n}{\eta}
  -B_\alpha^nM^{n+1}(e^n+\beta p^n),\nonumber\\
  \norm{g^n}
  &\leq\frac1{\underline\kappa}
  \left[
    \left(\eta^{-1}+\mu_{\max}L_B\right)\norm{e^n}
    +\beta\mu_{\max}L_B\norm{p^n}
  \right],\nonumber\\
  \mathcal D^n
  &\geq\frac{\mu_{\min}}{2\eta}
  \left(\norm{e^n}^2+(1-\beta^2)\norm{p^n}^2\right).
  \label{eq:gradient-dissipation-bounds}
\end{align}
The inequality \((a+b)^2\leq2a^2+2b^2\) gives
\(
\norm{g^n}^2\leq C_{\mathrm{mom}}\mathcal D^n
\).

For the direct update, \cref{eq:base-pbsav-step} gives
\begin{equation}
  \norm{g^n}
  \leq\frac{\eta^{-1}+L_B}{\underline\kappa}
  \norm{\Delta\theta^{n+1}},
  \qquad
  \mathcal D^n\geq
  \frac{\norm{\Delta\theta^{n+1}}^2}{\eta},
\end{equation}
and therefore
\(
\norm{g^n}^2\leq C_{\mathrm{dir}}\mathcal D^n
\).
Finally, summing \cref{eq:relaxed-energy-bound} and using
\cref{eq:strict-relaxation-budget} gives
\(
(1-\bar\rho)\sum_{n=0}^{N-1}\mathcal D^n\leq\Hcal^0
\).
This proves both estimates in \cref{eq:best-iterate-stationarity}; the
infinite-series claim follows by letting \(N\to\infty\).
\end{proof}

\section{A curvature condition for exact scalar tracking}
\label[appendix]{app:directional-coverage}

Full relaxation preserves exact scalar tracking when the correction covers
the true objective remainder along each step, up to a fixed fraction of the
inertial dissipation. Define
\begin{equation}
  R_F^n=F(\theta^{n+1})-F(\theta^n)
        -(g^n)^{\mathsf T}\Delta\theta^{n+1}
  \label{eq:true-objective-remainder}
\end{equation}
and the inertial part of the dissipation,
\begin{equation}
\begin{aligned}
  \mathcal D_{\mathrm{in}}^n
  ={}&\frac{\norm{p^n}_{M^n-M^{n+1}}^2}{2\eta}
  +\frac{\norm{p^{n+1}-\beta p^n}_{M^{n+1}}^2}{2\eta}\\
  &+\frac{1-\beta^2}{2\eta}\norm{p^n}_{M^{n+1}}^2.
\end{aligned}
\label{eq:inertial-dissipation}
\end{equation}
The true objective with kinetic energy is
\begin{equation}
  \E^n=F(\theta^n)+C+\frac{\norm{p^n}_{M^n}^2}{2\eta}.
  \label{eq:true-inertial-energy}
\end{equation}

\begin{proposition}
\label{prop:true-energy-identity}
At a consistent state \(q^n=Q^n\), the momentum update satisfies
\begin{equation}
  \E^n-\E^{n+1}
  =\mathcal D_{\mathrm{in}}^n
   +\norm{\Delta\theta^{n+1}}_{B_\alpha^n}^2-R_F^n.
  \label{eq:true-energy-identity}
\end{equation}
\end{proposition}

\begin{proof}
Take the inner product of \cref{eq:abstract-momentum-balance} with
\(\Delta\theta^{n+1}=M^{n+1}p^{n+1}\) and use weighted polarization.
With \(q^n/Q^n=1\), this gives
\[
  \frac{\norm{p^n}_{M^n}^2-\norm{p^{n+1}}_{M^{n+1}}^2}{2\eta}
  =\mathcal D_{\mathrm{in}}^n+(g^n)^{\mathsf T}\Delta\theta^{n+1}
   +\norm{\Delta\theta^{n+1}}_{B_\alpha^n}^2.
\]
Subtract the true objective increment to obtain
\cref{eq:true-energy-identity}.
\end{proof}

\begin{theorem}
\label{thm:coverage-full-relaxation}
Fix \(\eta>0\), \(0\leq\beta<1\), and \(0\leq\alpha\leq1\), and
let \(0\prec M^{n+1}\preceq M^n\). Initialize \(q^0=Q(\theta^0)\)
and use full relaxation \(\rho^n\equiv1\). Suppose that every step is
well defined and satisfies
\begin{equation}
  R_F^n\leq\norm{\Delta\theta^{n+1}}_{B_\alpha^n}^2
          +\gamma^n\mathcal D_{\mathrm{in}}^n,
  \qquad 0\leq\gamma^n\leq\bar\gamma<1.
  \label{eq:directional-coverage-condition}
\end{equation}
Then \(q^n=Q(\theta^n)\) for every \(n\), and
\begin{equation}
  \E^n-\E^{n+1}\geq(1-\bar\gamma)\mathcal D_{\mathrm{in}}^n\geq0.
  \label{eq:coverage-energy-law}
\end{equation}
If also \(\mu_{\min} I_d\preceq M^n\preceq\mu_{\max} I_d\) for
constants \(0<\mu_{\min}\leq\mu_{\max}<\infty\), and
\(\sup_n\norm{B_\alpha^n}<\infty\), then
\(p^n\to0\), \(\Delta\theta^{n+1}\to0\), and
\(\nabla F(\theta^n)\to0\).
\end{theorem}

\begin{proof}
Suppose \(q^n=Q^n\). Combining the modified and true energy identities
gives
\begin{equation}
  (\bar q^{n+1})^2+\mathcal D^n-Q(\theta^{n+1})^2
  =\mathcal D_{\mathrm{in}}^n
   +\norm{\Delta\theta^{n+1}}_{B_\alpha^n}^2-R_F^n\geq0.
  \label{eq:coverage-scalar-budget}
\end{equation}
Thus the minimum in \cref{eq:generic-relaxation} selects
\(Q(\theta^{n+1})^2\). Induction proves exact tracking, and
\cref{eq:true-energy-identity} yields \cref{eq:coverage-energy-law}.

Since \(\E^n\geq0\), summing the energy law gives
\(\sum_n\mathcal D_{\mathrm{in}}^n<\infty\). The damping and residual
terms, with the uniform lower mobility bound and \(\beta<1\), imply
\(p^n\to0\) and \(p^{n+1}-\beta p^n\to0\). The upper bound gives
\(\Delta\theta^{n+1}=M^{n+1}p^{n+1}\to0\). Finally,
\cref{eq:abstract-momentum-balance}, exact tracking, and the bounded
pullback imply \(g^n\to0\).
\end{proof}

For a quadratic objective, the condition with \(\gamma^n=0\) becomes
\[
  \frac12(\Delta\theta^{n+1})^{\mathsf T}H\Delta\theta^{n+1}
  \leq\norm{\Delta\theta^{n+1}}_{B_\alpha^n}^2.
\]
The correction must cover half the Hessian quadratic form along the step, not
dominate the Hessian in every direction. For nonlinear least squares, write
\(\nabla^2F(\theta^n)=G^n+A^n\), where \(G^n\) is the residual
Gauss--Newton matrix. If \(\norm{A^n}\leq\varepsilon^n\) and the Hessian
is Lipschitz along the step segment with constant \(L^n\), Taylor's formula
gives
\[
  R_F^n\leq\frac12\norm{\Delta\theta^{n+1}}_{G^n}^2
  +\frac{\varepsilon^n}{2}\norm{\Delta\theta^{n+1}}^2
  +\frac{L^n}{6}\norm{\Delta\theta^{n+1}}^3.
\]
Consequently, a sufficient condition is
\begin{equation}
  (\Delta\theta^{n+1})^{\mathsf T}
  \left(B_\alpha^n-\frac12G^n\right)\Delta\theta^{n+1}
  \geq\frac{\varepsilon^n}{2}\norm{\Delta\theta^{n+1}}^2
      +\frac{L^n}{6}\norm{\Delta\theta^{n+1}}^3.
  \label{eq:least-squares-coverage-margin}
\end{equation}
The projection identity \cref{eq:missed-block-curvature} describes the
curvature omitted by the least-squares components. Neither that identity nor
modified energy dissipation alone imposes
\cref{eq:directional-coverage-condition}, which can instead be checked from
the objective increment without forming a Hessian.

\bibliographystyle{elsarticle-num}
{\footnotesize\bibliography{references}}

@article{shen2018sav,
  author  = {Shen, Jie and Xu, Jie and Yang, Jiang},
  title   = {The scalar auxiliary variable ({SAV}) approach for gradient flows},
  journal = {Journal of Computational Physics},
  volume  = {353},
  pages   = {407--416},
  year    = {2018},
  doi     = {10.1016/j.jcp.2017.10.021}
}

@article{shen2019savreview,
  author  = {Shen, Jie and Xu, Jie and Yang, Jiang},
  title   = {A new class of efficient and robust energy stable schemes for gradient flows},
  journal = {SIAM Review},
  volume  = {61},
  number  = {3},
  pages   = {474--506},
  year    = {2019},
  doi     = {10.1137/17M1150153}
}

@article{shen2018convergence,
  author  = {Shen, Jie and Xu, Jie},
  title   = {Convergence and error analysis for the scalar auxiliary variable ({SAV}) schemes to gradient flows},
  journal = {SIAM Journal on Numerical Analysis},
  volume  = {56},
  number  = {5},
  pages   = {2895--2912},
  year    = {2018},
  doi     = {10.1137/17M1159968}
}

@article{cheng2018msav,
  author  = {Cheng, Qing and Shen, Jie},
  title   = {Multiple scalar auxiliary variable ({MSAV}) approach and its application to the phase-field vesicle membrane model},
  journal = {SIAM Journal on Scientific Computing},
  volume  = {40},
  number  = {6},
  pages   = {A3982--A4006},
  year    = {2018},
  doi     = {10.1137/18M1166961}
}

@article{huang2020newsav,
  author  = {Huang, Fukeng and Shen, Jie and Yang, Zhiguo},
  title   = {A highly efficient and accurate new scalar auxiliary variable approach for gradient flows},
  journal = {SIAM Journal on Scientific Computing},
  volume  = {42},
  number  = {4},
  pages   = {A2514--A2536},
  year    = {2020},
  doi     = {10.1137/19M1298627}
}

@article{jiang2022rsav,
  author  = {Jiang, Maosheng and Zhang, Zengyan and Zhao, Jia},
  title   = {Improving the accuracy and consistency of the scalar auxiliary variable ({SAV}) method with relaxation},
  journal = {Journal of Computational Physics},
  volume  = {456},
  pages   = {110954},
  year    = {2022},
  doi     = {10.1016/j.jcp.2022.110954}
}

@article{zhang2022gsavrelax,
  author  = {Zhang, Yanrong and Shen, Jie},
  title   = {A generalized {SAV} approach with relaxation for dissipative systems},
  journal = {Journal of Computational Physics},
  volume  = {464},
  pages   = {111311},
  year    = {2022},
  doi     = {10.1016/j.jcp.2022.111311}
}

@article{huang2024optimalrsav,
  author  = {Huang, Qiong-Ao and Yuan, Cheng and Zhang, Gengen and Zhang, Lian},
  title   = {A computationally optimal relaxed scalar auxiliary variable approach for solving gradient flow systems},
  journal = {Computers \& Mathematics with Applications},
  volume  = {156},
  pages   = {64--73},
  year    = {2024},
  doi     = {10.1016/j.camwa.2023.12.017}
}

@article{liu2022,
  author  = {Liu, Hailiang and Tian, Xuping},
  title   = {An adaptive gradient method with energy and momentum},
  journal = {Annals of Applied Mathematics},
  volume  = {38},
  number  = {2},
  pages   = {183--222},
  year    = {2022},
  doi     = {10.4208/aam.OA-2021-0095}
}

@article{liu2023savopt,
  author  = {Liu, Xinyu and Shen, Jie and Zhang, Xiangxiong},
  title   = {An Efficient and Robust Scalar Auxialiary Variable Based Algorithm for Discrete Gradient Systems Arising from Optimizations},
  journal = {SIAM Journal on Scientific Computing},
  volume  = {45},
  number  = {5},
  pages   = {A2304--A2324},
  year    = {2023},
  doi     = {10.1137/23M1545744}
}

@article{zhang2025rvsavopt,
  author  = {Zhang, Shiheng and Zhang, Jiahao and Shen, Jie and Lin, Guang},
  title   = {A relaxed vector auxiliary variable algorithm for unconstrained optimization problems},
  journal = {SIAM Journal on Scientific Computing},
  volume  = {47},
  number  = {1},
  pages   = {C126--C150},
  year    = {2025},
  doi     = {10.1137/23M1611087}
}

@misc{zhang2026pbsav,
  author        = {Zhang, Shiheng and Shen, Jie},
  title         = {Scalar-tracking {SAV} schemes with pullback corrections for gradient flows},
  year          = {2026},
  eprint        = {2606.18551},
  archivePrefix = {arXiv},
  primaryClass  = {math.NA},
  doi           = {10.48550/arXiv.2606.18551}
}

@article{polyak1964,
  author  = {Polyak, Boris T.},
  title   = {Some methods of speeding up the convergence of iteration methods},
  journal = {USSR Computational Mathematics and Mathematical Physics},
  volume  = {4},
  number  = {5},
  pages   = {1--17},
  year    = {1964},
  doi     = {10.1016/0041-5553(64)90137-5}
}

@misc{pooladzandi2022improving,
  author        = {Pooladzandi, Omead and Zhou, Yiming},
  title         = {Improving {Levenberg--Marquardt} algorithm for neural networks},
  year          = {2022},
  eprint        = {2212.08769},
  archivePrefix = {arXiv},
  doi           = {10.48550/arXiv.2212.08769}
}

@inproceedings{kingma2015adam,
  author    = {Kingma, Diederik P. and Ba, Jimmy},
  title     = {Adam: A method for stochastic optimization},
  booktitle = {International Conference on Learning Representations},
  year      = {2015},
  eprint    = {1412.6980},
  archivePrefix = {arXiv}
}

@inproceedings{reddi2018amsgrad,
  author    = {Reddi, Sashank J. and Kale, Satyen and Kumar, Sanjiv},
  title     = {On the convergence of {Adam} and beyond},
  booktitle = {International Conference on Learning Representations},
  year      = {2018},
  eprint    = {1904.09237},
  archivePrefix = {arXiv}
}

@inproceedings{loshchilov2019adamw,
  author    = {Loshchilov, Ilya and Hutter, Frank},
  title     = {Decoupled weight decay regularization},
  booktitle = {International Conference on Learning Representations},
  year      = {2019},
  eprint    = {1711.05101},
  archivePrefix = {arXiv}
}

@article{raissi2019pinn,
  author  = {Raissi, Maziar and Perdikaris, Paris and Karniadakis, George Em},
  title   = {Physics-informed neural networks: A deep learning framework for solving forward and inverse problems involving nonlinear partial differential equations},
  journal = {Journal of Computational Physics},
  volume  = {378},
  pages   = {686--707},
  year    = {2019},
  doi     = {10.1016/j.jcp.2018.10.045}
}

@article{karniadakis2021physics,
  author  = {Karniadakis, George Em and Kevrekidis, Ioannis G. and Lu, Lu and Perdikaris, Paris and Wang, Sifan and Yang, Liu},
  title   = {Physics-informed machine learning},
  journal = {Nature Reviews Physics},
  volume  = {3},
  number  = {6},
  pages   = {422--440},
  year    = {2021},
  doi     = {10.1038/s42254-021-00314-5}
}

@article{lu2021deeponet,
  author  = {Lu, Lu and Jin, Pengzhan and Pang, Guofei and Zhang, Zhongqiang and Karniadakis, George Em},
  title   = {Learning nonlinear operators via {DeepONet} based on the universal approximation theorem of operators},
  journal = {Nature Machine Intelligence},
  volume  = {3},
  number  = {3},
  pages   = {218--229},
  year    = {2021},
  doi     = {10.1038/s42256-021-00302-5}
}

@article{sherman1950,
  author  = {Sherman, Jack and Morrison, Winifred J.},
  title   = {Adjustment of an inverse matrix corresponding to a change in one element of a given matrix},
  journal = {The Annals of Mathematical Statistics},
  volume  = {21},
  number  = {1},
  pages   = {124--127},
  year    = {1950},
  doi     = {10.1214/aoms/1177729893}
}

@techreport{woodbury1950,
  author      = {Woodbury, Max A.},
  title       = {Inverting modified matrices},
  institution = {Statistical Research Group, Princeton University},
  number      = {Memorandum Report 42},
  year        = {1950}
}

@book{bjorck1996least,
  author    = {Bj{\"o}rck, {\AA}ke},
  title     = {Numerical Methods for Least Squares Problems},
  publisher = {Society for Industrial and Applied Mathematics},
  address   = {Philadelphia},
  year      = {1996},
  doi       = {10.1137/1.9781611971484}
}

@article{armijo1966,
  author  = {Armijo, Larry},
  title   = {Minimization of functions having {Lipschitz} continuous first partial derivatives},
  journal = {Pacific Journal of Mathematics},
  volume  = {16},
  number  = {1},
  pages   = {1--3},
  year    = {1966},
  doi     = {10.2140/pjm.1966.16.1}
}

@article{sobol1967,
  author  = {Sobol', Ilya M.},
  title   = {On the distribution of points in a cube and the approximate evaluation of integrals},
  journal = {USSR Computational Mathematics and Mathematical Physics},
  volume  = {7},
  number  = {4},
  pages   = {86--112},
  year    = {1967},
  doi     = {10.1016/0041-5553(67)90144-9}
}

@article{owen1997,
  author  = {Owen, Art B.},
  title   = {Scrambled net variance for integrals of smooth functions},
  journal = {The Annals of Statistics},
  volume  = {25},
  number  = {4},
  pages   = {1541--1562},
  year    = {1997},
  doi     = {10.1214/aos/1031594731}
}

@article{wang2021gradientpathologies,
  author  = {Wang, Sifan and Teng, Yujun and Perdikaris, Paris},
  title   = {Understanding and mitigating gradient flow pathologies in physics-informed neural networks},
  journal = {SIAM Journal on Scientific Computing},
  volume  = {43},
  number  = {5},
  pages   = {A3055--A3081},
  year    = {2021},
  doi     = {10.1137/20M1318043}
}

@article{wang2022pinnntk,
  author  = {Wang, Sifan and Yu, Xinling and Perdikaris, Paris},
  title   = {When and why {PINNs} fail to train: A neural tangent kernel perspective},
  journal = {Journal of Computational Physics},
  volume  = {449},
  pages   = {110768},
  year    = {2022},
  doi     = {10.1016/j.jcp.2021.110768}
}

@article{mcclenny2023selfadaptive,
  author  = {McClenny, Levi D. and Braga-Neto, Ulisses M.},
  title   = {Self-adaptive physics-informed neural networks},
  journal = {Journal of Computational Physics},
  volume  = {474},
  pages   = {111722},
  year    = {2023},
  doi     = {10.1016/j.jcp.2022.111722}
}

@article{ZHANGop1,
title = {A self-adaptive energy-based learning rate for stochastic gradient descent via Vector Auxiliary Variable method},
journal = {Engineering Applications of Artificial Intelligence},
volume = {160},
pages = {111731},
year = {2025},
issn = {0952-1976},
doi = {https://doi.org/10.1016/j.engappai.2025.111731},
author = {Jiahao Zhang and Christian Moya and Guang Lin}
}

@article{ZHANGop2,
title = {Muon with spectral guidance: Efficient optimization for scientific machine learning},
journal = {Journal of Computational Physics},
volume = {565},
pages = {115231},
year = {2026},
issn = {0021-9991},
doi = {https://doi.org/10.1016/j.jcp.2026.115231},
author = {Binghang Lu and Jiahao Zhang and Guang Lin}
}

\end{document}